\documentclass[twoside,11pt]{article}

\usepackage{amsmath,amsthm,amsbsy,amsfonts,comment,booktabs}
\usepackage{bm,graphicx,psfrag,epsf,booktabs,bbm,multirow}
\usepackage{enumerate}
\usepackage[toc,page]{appendix}
\usepackage[round, authoryear]{natbib}
\usepackage{url} % not crucial - just used below for the URL 
\usepackage{caption}
\usepackage{subcaption}
\usepackage{amssymb}
\usepackage{ifpdf}
\usepackage{listings}
\usepackage{setspace}
\usepackage{stmaryrd}
\RequirePackage[OT1]{fontenc}
\usepackage{algorithm}
\usepackage{algorithmic}
\usepackage[margin=1in]{geometry}
\usepackage{paralist}
\allowdisplaybreaks
\usepackage{amsbsy}

\usepackage[mathscr]{eucal}

\usepackage{bm}

\usepackage{xspace}

\mathchardef\mhyphen="2D

\ifpdf
\else
\newcommand{\href}[2]{{#2}}
\usepackage{url}
\fi

\newcommand{\Sec}[1]{\hyperref[sec:#1]{Section~\ref*{sec:#1}}} %section
\newcommand{\App}[1]{\hyperref[sec:#1]{Appendix~\ref*{sec:#1}}} %appendix
\newcommand{\Supp}[1]{\hyperref[sec:#1]{Supplement~\ref*{sec:#1}}} %supplement
\newcommand{\Eqn}[1]{\hyperref[eq:#1]{{\rm (\ref*{eq:#1})}}} %equation
\newcommand{\Part}[1]{\hyperref[part:#1]{(\ref*{part:#1})}} %part of theorem
\newcommand{\Fig}[1]{\hyperref[fig:#1]{Figure~\ref*{fig:#1}}} %figure
\newcommand{\Tab}[1]{\hyperref[tab:#1]{Table~\ref*{tab:#1}}} %table
\newcommand{\Thm}[1]{\hyperref[thm:#1]{Theorem~\ref*{thm:#1}}} %theorem
\newcommand{\Lem}[1]{\hyperref[lem:#1]{Lemma~\ref*{lem:#1}}} %lemma
\newcommand{\Prop}[1]{\hyperref[prop:#1]{Proposition~\ref*{prop:#1}}} %proposition
\newcommand{\Cor}[1]{\hyperref[cor:#1]{Corollary~\ref*{cor:#1}}} %corollary
\newcommand{\Def}[1]{\hyperref[def:#1]{Definition~\ref*{def:#1}}} %definition
\newcommand{\Alg}[1]{\hyperref[alg:#1]{Algorithm~\ref*{alg:#1}}} %algorithm
\newcommand{\Ex}[1]{\hyperref[ex:#1]{Example~\ref*{ex:#1}}} %example
\newcommand{\As}[1]{\hyperref[as:#1]{Assumption~{\rm\ref*{as:#1}}}} %assumption
\newcommand{\Reg}[1]{\hyperref[as:#1]{Condition~\ref*{reg:#1}}} %regularity condition
\newcommand{\AlgLine}[2]{\hyperref[alg:#1]{line~\ref*{line:#2} of Algorithm~\ref*{alg:#1}}}
\newcommand{\AlgLines}[3]{\hyperref[alg:#1]{lines~\ref*{line:#2}--\ref*{line:#3} of Algorithm~\ref*{alg:#1}}}

\newcommand{\Real}{\mathbb{R}}

\newtheorem{assumption}{Assumption}
\newcommand{\Kron}{\otimes} %Kronecker

\newcommand\extrafootertext[1]{%
	\bgroup
	\renewcommand\thefootnote{\fnsymbol{footnote}}%
	\renewcommand\thempfootnote{\fnsymbol{mpfootnote}}%
	\footnotetext[1]{#1}%
	\egroup
}

\usepackage[abbrvbib, preprint]{jmlr2e}

\usepackage{lastpage}
\ShortHeadings{Inertial Quadratic Majorization Minimization}{Heng and Wang}
\firstpageno{1}

\begin{document}

\title{Inertial Quadratic Majorization Minimization with Application to Kernel Regularized Learning}

\author{\name Qiang Heng \email heng.qiang6@gmail.com \\
     \addr School of Mathematics\\
       Southeast University\\
       Nanjing, China \\
\name Caixing Wang\footnote{} \email wangcaixing96@gmail.com \\
     \addr Department of Statistics\\
       The Chinese University of Hong Kong\\
       Shatin, New Territory, Hong Kong \\
       }
       
\extrafootertext{Corresponding author.}

\editor{My editor}

\maketitle

\begin{abstract}
First-order methods in convex optimization offer low per-iteration cost but often suffer from slow convergence, while second-order methods achieve fast local convergence at the expense of costly Hessian inversions. In this paper, we highlight a middle ground: minimizing a quadratic majorant with fixed curvature at each iteration. This strategy strikes a balance between per-iteration cost and convergence speed, and crucially allows the reuse of matrix decompositions, such as Cholesky or spectral decompositions, across iterations and varying regularization parameters. We introduce the Quadratic Majorization Minimization with Extrapolation (QMME) framework and establish its sequential convergence properties under standard assumptions. The new perspective of our analysis is to center the arguments around the induced norm of the curvature matrix $H$. To demonstrate practical advantages, we apply QMME to large-scale kernel regularized learning problems. In particular, we propose a novel Sylvester equation modelling technique for kernel multinomial regression. In Julia-based experiments, QMME compares favorably against various established first- and second-order methods. Furthermore, we demonstrate that our algorithms complement existing kernel approximation techniques through more efficiently handling sketching matrices with large projection dimensions. Our numerical experiments and real data analysis are available and fully reproducible at \url{https://github.com/qhengncsu/QMME.jl}.

\end{abstract}

\begin{keywords}
Quadratic Majorization, Momentum Extrapolation, Kernel Methods, Numerical Linear Algebra, Nyström Approximation
\end{keywords}

\section{Introduction}
In this article, we are primarily concerned with minimizing a smooth and convex function $f: \Real^n \rightarrow \Real$ for which the following quadratic majorization is globally available:
\begin{equation}\label{eq:quadmajor}
f(y) \le f(x) + \nabla f(x)^\top (y-x) + \frac{1}{2}(y-x)^\top H (y-x),
\end{equation}
where $H\in \Real^{n\times n}$ is a positive definite symmetric matrix. 

Distinct from the Newton-Raphson algorithm, $H$ here is a constant matrix and does not depend on the current iterate. If we make the further assumption that $f$ is twice-differentiable, then the matrix $H$ is an upper bound of the actual Hessian matrix $\nabla^2 f(x)$. Conversely, if we have $\nabla^2f \preceq H$, meaning $H-\nabla^2f$ is always positive semidefinite, then the quadratic upper bound \Eqn{quadmajor} follows from the integral
\begin{eqnarray*}
f(y) & = & f(x) + \nabla f(x)^\top(y-x) \nonumber \\
&    & +\,(y-x)^\top \! \int_0^1 \nabla ^2f[x+t(y-x)](1-t)\,dt\,(y-x). \nonumber \\
\end{eqnarray*}
In particular, if $\nabla f$ is $L$-Lipschitz and $H = L I_n$ ($I_n$ is the identity matrix), then \Eqn{quadmajor} represents the well-known quadratic upper bound for functions with Lipschitz continuous gradient.  In this paper, we are interested in the more general case where $H$ is non-diagonal and contains curvature information of $f$. Leveraging the quadratic majorization \Eqn{quadmajor}, the majorant $h(x)$ of $f(x)$ at $x=x^k$ writes as
\begin{equation}\label{eq:overallmajor}
f(x)\le h(x|x^k) = f(x^k)+ \nabla f(x^k)^\top (x-x^k) + \frac{1}{2}(x-x^k)^\top H (x-x^k).
\end{equation}
The iteration $x^{k+1} = \text{argmin}_x \;h(x|x^k)$ is an instance of majorization minorization (MM) algorithm \citep{hunter2004atutorial} that ensures monotonic descent of $f(x^k)$.

Since $H$ remains constant, matrix decompositions such as Cholesky or spectral decompositions can be recycled across all iterations. Such an MM algorithm may not achieve the superlinear rates of convergence of Newton methods \citep{Nesterov2006CubicRO,Doikov2022SuperUniversalRN}, but the computational savings from recycling the same matrix decomposition can produce highly competitive algorithms for many important problems in statistics \citep{heng2025tactics}. Since $H$ is only an approximation of the actual Hessian matrix, the quadratic MM algorithm typically requires many more iterations than Newton-Raphson, although it typically converges much faster than conventional first-order methods. Fortunately, the quadratic MM framework is also amenable to acceleration by extrapolation. We define Quadratic Majorization Minorization with Extrapolation (QMME) as the following iteration:
\begin{eqnarray}\label{eq:MMextrapolate}
y^k & = &x^k + \beta_k (x^k - x^{k-1}),\nonumber\\
x^{k+1} & = &\underset{x\in \Real^n}{\text{argmin}} \quad \langle \nabla f(y^k) , x\rangle +\frac{1}{2}\| y^k - x \|_H^2,
\end{eqnarray}
where $\beta_k\in [0,1)$ is the extrapolation coefficient and $\|x\|_H = \sqrt{x^\top Hx}$. We note that selecting an appropriate sequence of $\beta_k$ is a vital issue for the practical efficiency of QMME. 

\subsection{Contributions}
We propose the Quadratic Majorization-Minimization with Extrapolation (QMME) algorithm for convex optimization. To adapt extrapolation to the objective's local geometry while preserving global convergence, we adopt a hybrid restart strategy \citep{ODonoghue2012AdaptiveRF,Wen2018APD} that combines periodic and adaptive restarts. By establishing the descent property of a potential function similar to the one in  \citet{wen2017linear} and \citet{Wen2018APD}, we establish subsequential convergence of our algorithm when extrapolation coefficients are bounded away from 1 and prove global convergence under the Kurdyka–Łojasiewicz (KL) property \citep{bolte2007lojasiewicz,attouch2009convergence,attouch2010proximal,attouch2013convergence}. The novel aspect of our analysis is to leverage the scaled norm induced by the curvature matrix $H$. Using the same technique, we also explore connections with inertial Krasnoselskii–Mann iterations \citep{maulen2024inertial}. With a few technical variations, we generalize our algorithm and theory to the scenario of composite optimization in the appendix. We chose to focus on smooth optimization in the main paper since composite optimization deviates from the theme of recycling matrix decompositions and also kernel regularized learning. To avoid diluting the main message of the paper, we present the algorithm and theory for composite optimization as an extension.

To demonstrate QMME's efficacy, we apply it to a key problem in statistical machine learning, namely kernel regularized learning with random sketching \citep{scholkopf2002learning,mendelson2010regularization,yang2017randomized}. The examples we consider include kernel-smoothed quantile regression \citep{li2007quantile,wang2024optimal}, kernel logistic regression \citep{keerthi2005fast}, and kernel multinomial regression \citep{zhu2005kernel,karsmakers2007multi}. For each, we derive Hessian upper bounds and present normal equations for minimizing the quadratic majorant. Notably, for kernel multinomial regression, we introduce a novel Sylvester equation technique that greatly reduces per-iteration complexity. We describe multinomial regression models using both the standard parameterization (coefficient matrix $X$ has $q-1$ columns) and the full parameterization  (coefficient matrix $X$ has $q$ columns). Our MM algorithms recycle matrix factorizations (Cholesky, spectral, Schur) across iterations and sometimes across regularization parameters. For kernel quantile and logistic regression, QMME outperforms first- and second-order methods when projection dimension is sufficiently large. For kernel multinomial regression, the advantage of QMME over existing methods is decisive, especially when the number of classes $q$ is large. 

To summarize succinctly, our contributions are:
\begin{itemize}
\item Theoretically, we integrate a scaled norm analysis technique with the optimization framework of \cite{wen2017linear} and \cite{Wen2018APD} to investigate the convergence properties of a quadratic MM algorithm incorporating extrapolation. This integration allows us to establish subsequential convergence in general settings and global convergence under the Kurdyka–Łojasiewicz (KL) property. By drawing connections to inertial Krasnoselskii–Mann iterations \citep{maulen2024inertial}, we can relax the KL assumption by imposing additional restrictions on the extrapolation step size $\beta_k$. We also generalize our algorithm and theory to the case of composite optimization in the appendix.

\item Algorithmically, we develop efficient new algorithms for kernel regularized learning with random sketching that recycle matrix decompositions. By scaling more effectively to larger projection dimensions, our optimization algorithms complement existing kernel approximation techniques such as Nyström approximation \citep{williams2000using,rudi2015less}. In particular, for kernel multinomial regression, we devise a novel Sylvester equation technique that greatly reduces the normal equation dimensionality, which results in an optimization algorithm that convincingly outperforms existing alternatives. Beyond the standard parameterization of multinomial regression, we also investigate a full parameterization that eliminates the need to select a reference category while often achieving higher log-likelihoods.
\end{itemize}

\subsection{Related Works}
\paragraph{Majorization Minorization}
The Majorization-Minimization (MM) principle \citep{hunter2004atutorial,lange2000optimization} provides a flexible and powerful framework for deriving optimization algorithms, encompassing well-known approaches like the EM algorithm and forward-backward splitting \citep{bauschke2011convex} as special cases. In nonconvex multiplicative models (e.g., matrix/tensor factorization), blockwise variable updates are common, leading to the family of block MM algorithms \citep{Xu2013ABC, Hien2020AnIB, Hien2021BlockBM,  Lyu2025BlockMW}. Within MM, the idea of quadratic upper bound and matrix decomposition reuse originated in early statistics literature \citep{heiser1987correspondence, bohning1992multinomial, kiers1997weighted} and was revisited recently \citep{heng2025tactics}. While these strategies yield practically fast algorithms, their convergence properties lack rigorous examination, especially in the presence of an extrapolation technique. \citet{Robini2024TheAO} explicitly considered quadratic surrogates but focused on nonconvex objectives without incorporating extrapolation or reusing matrix decompositions.

\paragraph{Inertial Extrapolation} The idea of inertial extrapolation in optimization was first introduced by \citet{Polyak1964SomeMO}, who proposed accelerating the convergence of gradient descent by incorporating an inertial (momentum-like) term. This led to the development of the heavy-ball method, characterized by the update rule: $x^{k+1} = x^k - \alpha_k \nabla f(x^k)+ \beta_k(x^k-x^{k-1})$.  Subsequent seminal works by Nesterov \citep{Nesterov1983AMF, Nesterov2005SmoothMO, Nesterov2014IntroductoryLO} introduced the Nesterov accelerated gradient method, which differs from the heavy-ball approach by applying extrapolation before computing the gradient. The Nesterov updates are given by: $y^k = x^k + \beta_k(x^k-x^{k-1}), \; x^{k+1} = y^k - \alpha_k \nabla f(y^k)$. This method achieves an improved convergence rate $O(1/k^2)$ in function values for smooth convex functions. The same rate was later extended to composite convex optimization problems via the Fast Iterative Shrinkage-Thresholding Algorithm (FISTA) \citep{Beck2009AFI}. \citet{ODonoghue2012AdaptiveRF} demonstrated that a simple heuristic, adaptive restart, can significantly enhance the convergence of accelerated gradient methods. Later, under the assumption that the extrapolation coefficients are bounded away from 1, \citet{wen2017linear} showed that extrapolated proximal gradient algorithms always exhibit subsequential convergence, and under suitable conditions, enjoy R-linear convergence. The key principle of their analysis is the construction of a monotonically decreasing potential function. For proximal difference-of-convex problems, \cite{Wen2018APD} established global convergence of a similar algorithm under the Kurdyka–Łojasiewicz (KL) property. 

\paragraph{Kernel Regularized Learning and Computational Strategies} Kernel methods \citep{vapnik1999nature,wahba1990spline,hofmann2008kernel} have become a cornerstone of statistical machine learning, enabling the modeling of complex relationships in high-dimensional data via the kernel trick. Kernel regularized learning \citep{micchelli2005learning,mendelson2010regularization}, which entails minimizing a regularized loss function over a reproducing kernel Hilbert space (RKHS) $\mathcal{H}$ of real-valued functions $f$, has been widely applied in both classification \citep{scholkopf2002learning,steinwart2008support} and regression \citep{hastie2005elements,caponnetto2007optimal}.

Despite their flexibility, kernel methods suffer from high computational costs, primarily due to the need to compute, store, and invert large kernel matrices. In the case of kernel regularized least squares, the computational complexity scales as $\mathcal{O}(n^3)$, with storage requirements of $\mathcal{O}(n^2)$, where $n$ denotes the sample size. For general kernel regularized learning problems lacking closed-form solutions, the computational burden is even greater. To mitigate these costs, several approximation techniques have been proposed, including the Nyström method \citep{williams2000using,rudi2015less}, random features \citep{rahimi2007random,rudi2017generalization,wang2024communication}, and sketching methods \citep{yang2017randomized,lian2021distributed}. These approaches aim to reduce problem dimensionality while preserving essential properties of the kernel.

\subsection{Notation and Organization}

We denote by $\mathbb{R}^n$ the $n$-dimensional Euclidean space equipped with the inner product $\langle \cdot, \cdot \rangle$. The $\ell_2$-norm is denoted by $\|\cdot\|$. For a symmetric positive definite matrix $H$, its induced norm is defined as $\|x\|_H = \sqrt{x^\top H x}$. For any matrix $A \in \mathbb{R}^{m \times n}$, $\sigma_{\max}(A)$ and $\sigma_{\min}(A)$ represent its largest and smallest singular values, respectively. Given a non-empty closed set $\mathcal{C} \subset \mathbb{R}^n$, the distance from a point $x \in \mathbb{R}^n$ to $\mathcal{C}$ is defined as $\text{dist}(x, \mathcal{C}) = \inf_{y \in \mathcal{C}} \|x - y\|$. An extended real-valued function $h: \mathbb{R}^n \to [-\infty, \infty]$ is proper if it never attains $-\infty$ and its domain $\text{dom}\; h = \{x \in \mathbb{R}^n \mid h(x) < \infty\}$ is non-empty. A proper function is closed if it is lower semicontinuous, that is, for all $x \in \mathbb{R}^n$ and every sequence ${x^k} \to x$, we have $h(x) \leq \underset{k \to \infty}{\liminf} \;h(x^k)$. The level sets $\{x \in \mathbb{R}^n \mid h(x) \leq \alpha\}$ of a closed function are closed for all $\alpha \in \mathbb{R}$. The subdifferential of a proper closed function $h: \Real^n \rightarrow \Real\cup\{\infty\}$ at $x\in \text{dom} \;h $ is given by:
\begin{eqnarray*}
\partial h(x) & = & \{v\in \Real^n|\exists \;x^k \overset{h} \rightarrow x,\;v^k \rightarrow v, s.t. \;\underset{y \rightarrow x^k}{\liminf} \;\frac{h(y) - h(x^k) -\langle v^k, y-x^k \rangle }{\|y -x^k\| } \text{ for each} \;k \},
\end{eqnarray*}
where $x^k \overset{h} \rightarrow x$ means $x^k \rightarrow x$ and $h(x^k) \rightarrow h(x)$. When $h$ is convex, the subdifferential at $x \in \mathbb{R}^n$ reduces to
\begin{eqnarray*}
\partial h(x) & = &\{v\in \Real^n: h(u) -h(x) - \langle v, u-x \rangle\ge 0,\;\forall u \in \Real^n \}.
\end{eqnarray*}
If $h$ is continuously differentiable at $x$, then $\partial h(x) = \{ \nabla h(x) \}$. 

The remainder of the paper is organized as follows. In \Sec{algorithm}, we present our main algorithm and establish its sequential convergence properties under appropriate assumptions. The quadratic MM algorithms for solving specific instances of kernel regularized learning problems are detailed in \Sec{kernel}. In \Sec{ne}, we demonstrate the advantages of our quadratic MM algorithms over existing first- and second-order methods using carefully designed numerical experiments. A DNA codon usage real data example for kernel multinomial regression is provided in \Sec{RDA}. We conclude with a discussion in \Sec{discuss}. Additional proofs and a generalization to composite optimization are provided in the appendix.

\section{Algorithm and Convergence Analysis}\label{sec:algorithm}
\begin{algorithm}[htbp]
\caption{Quadratic Majorization Minorization with Extrapolation (QMME)}\label{alg:QMME}
\begin{algorithmic}[1]
\REQUIRE Initial iterates $x_{1}=x_0\in \Real^n$, $l=1$, maximum restart period $P$.
\FOR{$k=1,2,\dots$}
\STATE $\beta_k=\frac{l}{l+2}$.
\STATE $y^k = x^k+\beta_k(x^k-x^{k-1})$.
\STATE $x^{k+1} = y^k - H^{-1}\nabla f(y^k)$.
\STATE $l\leftarrow l+1$.
\IF{$f(x^{k+1})>f(x^{k})$ or $l=P$}
\STATE Reset $l=1$.
\ENDIF
\ENDFOR
\end{algorithmic}
\end{algorithm}
We introduce the Quadratic Majorization Minorization with Extrapolation (QMME) algorithm in \Alg{QMME}. We set the extrapolation coefficient as $\beta_k=\frac{l}{l+2}$ where $l$ starts at 1 and increases by 1 after each iteration. Whenever the function value breaks the descent property or the momentum counter $l$ hits a maximum value of $P$, we reset $l$ to 1. Thus, it should be clear from the algorithm setup that $\sup_k\beta_k \le \frac{P-1}{P+1}<1$. We start our convergence analysis with the following blanket assumptions and two important lemmas that will be used throughout this section.

\begin{assumption}\label{as:blanket}
The function $f: \Real^n\rightarrow \Real$ is convex, differentiable, and lower bounded. The majorization \Eqn{quadmajor} holds for any $x,y\in \Real^n$. The optimal set $\mathcal{X} =\{x\in \Real^n|\nabla f(x) = 0 \} $ is non-empty.
\end{assumption}
\begin{lemma}\label{lem:normequiv}
We have the following two inequalities from basic linear algebra
\begin{eqnarray*}
\sqrt{\sigma_{\min}(H)}\| x\|\le \sqrt{x^\top H x} \le \sqrt{\sigma_{\max}(H)}\| x\|,
\end{eqnarray*}
\begin{eqnarray*}
\frac{1}{\sqrt{\sigma_{\max}(H)}}\| x\|\le \sqrt{x^\top H^{-1} x} \le \frac{1}{\sqrt{\sigma_{\min}(H)}}\| x\|,
\end{eqnarray*}
where $\sigma_{\min}(H)$ and $\sigma_{\max}(H)$ are the smallest and largest singular values of $H$. 
\end{lemma}
\begin{lemma}\label{lem:Lipschitz}
The quadratic majorization \Eqn{quadmajor} implies Lipschitz continuity of $\nabla f$. 
\end{lemma}
\begin{proof}
 A standard result is that for a convex and differentiable function, the gradient is Lipschitz continuous with constant $L$ if and only if: 
\begin{eqnarray*}
 \forall \;\; x,y\in \Real^n, &&f(y) \le f(x) + \nabla f(x)^\top (y-x) + \frac{L}{2}\|y-x\|^2. 
\end{eqnarray*}
 Due to the quadratic majorization, we have
\begin{eqnarray*}
 f(y) & \le & f(x) + \nabla f(x)^\top (y-x) + \frac{1}{2}\|y-x\|_H^2 \\
 & \le & f(x) + \nabla f(x)^\top (y-x) + \frac{1}{2}\sigma_{\max}(H)\|y-x\|^2.
\end{eqnarray*}
Thus $\nabla f$ is $\sigma_{\max}(H)$-Lipschitz. 
\end{proof}

\begin{remark}
\Lem{normequiv} establishes the equivalence between the Euclidean norm $\|\cdot\|$ and the scaled norm $\|\cdot\|_H$. Convergence in either norm implies convergence in the other. \Lem{Lipschitz} also admits a converse statement: for any convex function with $L$-Lipschitz continuous gradient, the quadratic majorization \Eqn{quadmajor} holds at least for $H = L I_n$.    
\end{remark}

Under \As{blanket}, we next show that every convergent subsequence of the sequence generated by \Alg{QMME} converges to a stationary (and consequently optimal) point of $f$. 
\begin{lemma}\label{lem:Hmajorize}
The function $f$ can be written as the difference of two functions $f_1$ and $f_2$, where $f_1$ and $f_2$ satisfies the following conditions:\\
(i) $f_1$ is convex. \\
(ii) For any $x,y\in \Real^n$, we have $f_1(y) \le f_1(x) + \langle \nabla f_1(x), y - x \rangle + \frac{1}{2}\| y-x\|_H^2$. \\
(iii) For any $x, y\in \Real^n$, we have $f_2(y) \le f_2(x) + \langle \nabla f_2(x), y - x \rangle + \frac{1}{2}\| y-x\|_{\epsilon H}^2$ where $0<\epsilon < 1$. 
\end{lemma}
\begin{proof}
Straightforward by letting $f_1 = f$ and $f_2 = 0$.
\end{proof}

\begin{lemma}\label{lem:strongconvexity}
We have for any $z\in \Real^n$,
\begin{eqnarray}\label{eq:strongconvexity}
&& \langle \nabla f(y^k),x^{k+1}\rangle + \frac{1}{2} \| y^k - x^{k+1} \|_H^2 +\frac{1}{2} \| z - x^{k+1} \|_H^2    \nonumber\\& = & \langle \nabla f(y^k) , z\rangle +\frac{1}{2}\| y^k - z \|_H^2.
\end{eqnarray}
\end{lemma}
\begin{proof}
It follows from basic linear algebra that
\begin{eqnarray}\label{eq:expand}
\frac{1}{2}\|y^k -z\|_H^2 & = & \frac{1}{2}\|y^k - x^{k+1}+x^{k+1}-z\|_H^2\nonumber \\
 & = & \frac{1}{2}\| y^k - x^{k+1}\|_H^2 + \langle H (y^k -x^{k+1}), x^{k+1} -z \rangle + \frac{1}{2}\| x^{k+1}-z\|_H^2,
\end{eqnarray}
Due to the optimality condition of the minimization problem \Eqn{MMextrapolate}, we have
\begin{eqnarray*}
\nabla f (y^k) + H(x^{k+1} - y^k) = 0,
\end{eqnarray*}
Plug in $\nabla f (y^k) = H(y^k-x^{k+1})$ into \Eqn{expand}, we have
\begin{eqnarray*}\label{eq:rearrange}
\frac{1}{2}\|y^k -z\|_H^2 = \langle - \nabla  f (y^k), z - x^{k+1}\rangle + \frac{1}{2} \| y^k - x^{k+1} \|_H^2 + \frac{1}{2}\| x^{k+1}-z\|_H^2.
\end{eqnarray*}
Moving $\langle - \nabla  f (y^k), z\rangle$ to the left gives us \Eqn{strongconvexity}.
\end{proof}

\begin{lemma}\label{lem:aux}
Suppose that $\frac{1+\epsilon^2}{2}\beta_k^2 <\frac{1}{2}$ for all $k$, the auxillary sequence $E_k = f(x^k)+\frac{1}{2} \|x^k - x^{k-1}\|_H^2$ is non-increasing and convergent.
\end{lemma} 
\begin{proof}
Due to the global quadratic majorization, 
\begin{eqnarray}\label{eq:major}
f(x^{k+1}) \le f(y^k) + \langle \nabla f(y^k), x^{k+1} - y^k \rangle + \frac{1}{2}\| x^{k+1}-y^k\|_H^2.
\end{eqnarray}
Summing \Eqn{strongconvexity} and \Eqn{major} gives us 
\begin{eqnarray}\label{eq:boundnext}
f(x^{k+1}) \le f(y^k) + \langle \nabla f(y^k), z - y^k \rangle + \frac{1}{2}  \| z-y^k \|_H^2 - \frac{1}{2} \|z - x^{k+1} \|_H^2.
\end{eqnarray}
Recall that $f= f_1 -f_2$, where $f_1$ is convex, while $f_2$ satisfies the inequality in (iii) of \Lem{Hmajorize},  so that
\begin{eqnarray*}
f_1(y^k) + \langle \nabla f_1(y^k), z - y^k  \rangle & \le & f_1(z),\\
f_2(z) - f_2(y^k) - \langle \nabla f_2(y^k), z - y^k\rangle & \le & \frac{\epsilon^2}{2} \|z-y^k\|_H^2,
\end{eqnarray*}
adding the two inequalities yields
\begin{eqnarray}\label{eq:linearmajor}
f(y^k) +  \langle \nabla f(y^k), z - y^k \rangle \le f(z) + \frac{\epsilon^2}{2} \|z-y^k\|_H^2.
\end{eqnarray}
Summing \Eqn{boundnext} and \Eqn{linearmajor} gives us
\begin{eqnarray}\label{eq:boundf}
f(x^{k+1})  \le f(z) +\frac{1+\epsilon^2}{2} \|z - y^k \|_H^2 - \frac{1}{2} \|z - x^{k+1} \|_H^2.
\end{eqnarray}
Plug in $z=x^k$ and $y-x^k = \beta_k(x^k - x^{k-1})$, we have
\begin{eqnarray}\label{eq:difff}
f(x^{k+1}) - f(x^k) \le \frac{1+\epsilon^2}{2} \beta_k^2 \|x^k - x^{k-1} \|_H^2 - \frac{1}{2} \|x^{k+1} - x^k\|_H^2.
\end{eqnarray}
Subsequently
\begin{eqnarray}\label{eq:diffofE}
E_{k+1} - E_{k} & = & f(x^{k+1})+\frac{1}{2}\|x^{k+1} - x^{k}\|_H^2 - f(x^k)  - \frac{1}{2} \|x^{k} - x^{k-1}\|_H^2 \nonumber \\
& \le &  (\frac{1+\epsilon^2}{2}\beta_k^2 - \frac{1}{2}) \|x^{k} - x^{k-1}\|_H^2. 
\end{eqnarray}
The constants $\frac{1+\epsilon^2}{2}\beta_k^2 - \frac{1}{2}$ are negative by assumption. Thus, the sequence $E_{k}$ is non-increasing. Since $f$ is bounded from below, $E_{k}$ is also bounded from below and thus convergent.
\end{proof}

\begin{theorem}\label{thm:subsequenceconverge}
Suppose that $\bar{\beta} = \sup\{\beta_k\} < \frac{1}{\sqrt{1+\epsilon^2}}$ and $\{x^k\}$ is a sequence generated by \Eqn{MMextrapolate}, then the following statements hold. \\
(i) $\sum_{k=0}^\infty \|x^{k+1} - x^k \|_H^2<\infty$.\\
(ii) any accumulation point of $\{x^k\}$ is a stationary point of $f$.
\end{theorem}
\begin{proof}
We have
\begin{eqnarray*}
E_{k+1} - E_{k} \le (\frac{1+\epsilon^2}{2}\beta_k^2 - \frac{1}{2}) \|x^{k} - x^{k-1}\|_H^2 \le (\frac{1+\epsilon^2}{2}\bar{\beta}^2 - \frac{1}{2}) \|x^{k} - x^{k-1}\|_H^2.
\end{eqnarray*}
Since $\bar{\beta}< \frac{1}{\sqrt{1+\epsilon^2}}$, telescoping gives us
\begin{eqnarray*}
0 & \le &  \sum_{k=1}^N (\frac{1}{2} - \frac{1+\epsilon^2}{2}\bar{\beta}^2) \|x^{k} - x^{k-1}\|_H^2 \le E_{1} - E_{N+1}
\end{eqnarray*}
Since $E_{k}$ is convergent, the summation $\sum_{k=0}^\infty \|x^{k+1} - x^k \|_H^2$ is finite. Let $\bar{x}$ be an accumulation point, namely there exists a subsequence $\{x^{k_i}\}$ such that $\lim_{{i\rightarrow \infty}} x^{k_i} = \bar{x}$. Due to the optimality condition of \Eqn{MMextrapolate}, 
\begin{eqnarray*}
\nabla f(y^{k_i}) + H (x^{k_i+1} - y^{k_i}) = 0 . 
\end{eqnarray*}
Combined the definition of $y^{k_i}$, we have
\begin{eqnarray*}
-H (x^{k_i+1} - x^{k_i}-\beta_k (x^{k_i} - x^{k_i-1})) = \nabla f(y^{k_i}) 
\end{eqnarray*}
Taking $i$ to the limit on both sizes, due to the Lipschitz continuity of $\nabla f$, we have $\nabla f(\bar{x}) =0$.
Thus, any accumulation point of $\{x^k\}$ is a stationary point for $f$.
\end{proof}

\Thm{subsequenceconverge} establishes subsequential convergence of \Alg{QMME}. Before proceeding to the analysis of global convergence, we establish the following simple fact, which will be useful later.

\begin{proposition}\label{prop:constantcluster}
Let $\{x^k\}$ be the sequence generated by \Alg{QMME}. Then we have the following:
\\(i) $\zeta:= \underset{k\rightarrow \infty}\lim f(x^k)$ exists.
\\(ii) $f\equiv \zeta$ on $\Omega$, where $\Omega$ is the set of accumulation points for $\{x^k\}$. 
\end{proposition}
\begin{proof}
Recall that $E^k = f(x^k)+\frac{1}{2}\|x^k-x^{k-1}\|_H^2$ is monotonically decreasing and convergent. In the meantime, due to (ii) of \Thm{subsequenceconverge}, $\frac{1}{2}\|x^k-x^{k-1}\|_H^2$ converges to 0. Thus $\zeta:= \underset{k\rightarrow \infty}\lim f(x^k)$ exists and equals $\underset{k\rightarrow \infty}\lim E^k$. For any $\hat{x}\in \Omega$, there is a subsequence $\{x^{k_i}\}$ such that $x^{k_i}\rightarrow \hat{x}$. We have $f(\hat{x})=\zeta$ due to the Lipschitz continuity of $f$.
\end{proof}

Next, we study the global convergence properties of $\{x^k\}$. We first make the additional assumption that the optimal set $\mathcal{X}$ is bounded.

\begin{assumption}\label{as:bound}
The optimal set $\mathcal{X}$ is bounded. 
\end{assumption}
We note that this further implies for all $\zeta \ge \inf_x \; f(x)$, the level set $\{x\in \Real^n|f(x)\le \zeta\}$ is bounded (see corollary 8.7.1 of \cite{rockafellar1997convex}). We next introduce the notion of Kurdyka-Lojasiewicz (KL) property \citep{bolte2007lojasiewicz,attouch2009convergence,attouch2010proximal,attouch2013convergence}.

\begin{definition}[KL Property]\label{def:KL}
A proper closed function $h$ is said to satisfy the \emph{Kurdyka--Łojasiewicz (KL) property} at a point $\hat{x} \in \text{dom} \, \partial h$ if there exist a constant $a \in (0,\infty]$, a neighborhood $\mathcal{O}$ of $\hat{x}$, and a continuous, concave function $\phi : [0,a) \to \Real_+$ with $\phi(0) = 0$, such that the following conditions hold:
\begin{itemize}
    \item[(i)] $\phi$ is continuously differentiable on $(0,a)$ and satisfies $\phi' > 0$;
    \item[(ii)] For all $x \in \mathcal{O}$ with $h(\hat{x}) < h(x) < h(\hat{x}) + a$, we have
    \[
    \phi'\left(h(x) - h(\hat{x})\right) \cdot \text{dist}(0, \partial h(x)) \geq 1.
    \]
\end{itemize}
A function $h$ that satisfies the KL property at every point in $\text{dom} \, \partial h$ is called a \emph{KL function}.
\end{definition}

\begin{lemma}[Uniform KL Property]\label{lem:uniformKL}
Let $h$ be a proper closed function, and let $\Gamma$ be a compact subset of $\text{dom} \, \partial h$ on which $h$ is constant. If $h$ satisfies the KL property at every point in $\Gamma$, then there exist constants $\epsilon > 0$ and $a > 0$, and a function $\phi$ satisfying the conditions in Definition~\ref{def:KL}, such that for all $\hat{x} \in \Gamma$ and all $x$ with $\text{dist}(x,\Gamma) < \epsilon$ and $h(\hat{x}) < h(x) < h(\hat{x}) + a$, we have
\[
\phi'\left(h(x) - h(\hat{x})\right) \cdot \text{dist}(0, \partial h(x)) \geq 1.
\]
\end{lemma}
We define the auxilliary function $E(x,y)= f(x)+\frac{1}{2}\|x-y\|_H^2$. We also introduce some useful lemmas before diving into the main global convergence result. The proof \Lem{keylemma1} was first given in our prior work \cite{heng2025tactics}. We include the proof here for completeness.
\begin{lemma}\label{lem:keylemma1}
For any $x,y\in \Real^n$, we have
\begin{eqnarray*}
 [\nabla f(x) - \nabla f(y)]^\top H^{-1}[\nabla f(x) - \nabla f(y)]
 & \le  & [\nabla f(x)-\nabla f(y)]^\top (x - y).
\end{eqnarray*}
\end{lemma}
\begin{proof}
Substituting $y = x - H^{-1}\nabla f(x)$ into \Eqn{quadmajor} yields:
\begin{eqnarray} \label{eq:kl}
\inf_y f(y) & \leq & f(x) - \frac{1}{2} \nabla f(x)^\top H^{-1} \nabla f(x).
\end{eqnarray}

Define the auxiliary function $g_x(y) = f(y) - \nabla f(x)^\top y$. This function is convex, achieves its minimum at $y = x$, and satisfies analogous inequalities to \Eqn{kl}. Applying \eqref{eq:kl} to $g_x$ gives:
\begin{eqnarray*}
g_x(y) - g_x(x) & \geq & \frac{1}{2} \nabla g_x(y)^\top H^{-1} \nabla g_x(y) \\
& = & \frac{1}{2} [\nabla f(x) - \nabla f(y)]^\top H^{-1} [\nabla f(x) - \nabla f(y)].
\end{eqnarray*}
Since $g_x(y) - g_x(x) = f(y) - f(x) - \nabla f(x)^\top (y - x)$, we have:
\begin{eqnarray} \label{ineq1}
f(y) - f(x) - \nabla f(x)^\top (y - x) & \geq & \frac{1}{2} [\nabla f(y) - \nabla f(x)]^\top H^{-1} [\nabla f(y) - \nabla f(x)].
\end{eqnarray}

By symmetry (swapping $x$ and $y$), we obtain:
\begin{eqnarray} \label{ineq2}
f(x) - f(y) - \nabla f(y)^\top (x - y) & \geq & \frac{1}{2} [\nabla f(y) - \nabla f(x)]^\top H^{-1} [\nabla f(y) - \nabla f(x)].
\end{eqnarray}
Adding \eqref{ineq1} and \eqref{ineq2} gives:
\begin{eqnarray*}
[\nabla f(y) - \nabla f(x)]^\top H^{-1} [\nabla f(y) - \nabla f(x)] & \leq & [\nabla f(y) - \nabla f(x)]^\top (y - x).
\end{eqnarray*}
\end{proof}
\begin{lemma}\label{lem:keylemma2}
For any $x,y\in \Real^n$, we have
\begin{eqnarray*}
(\nabla f(x)-\nabla f(y))^\top (x-y) \le (y-x)^\top H (y-x).
\end{eqnarray*}
\end{lemma}
\begin{proof}
Due to the global quadratic majorization, we have for any $x,y\in \Real^n$,
\begin{eqnarray*}
f(y) & \le & f(x) + \nabla f(x)^\top (y-x) + \frac{1}{2}(y-x)^\top H (y-x),\\
f(x) & \le & f(y) + \nabla f(y)^\top (x-y) + \frac{1}{2}(y-x)^\top H (y-x).
\end{eqnarray*}
Adding the two inequalities yields the result.
\end{proof}

\begin{theorem}\label{thm:globalconverge}
Under \As{bound}, if $f$ is a KL function, then the sequence $\{x^k\}$ generated by \Alg{QMME} satisfy the following statements:\\
(i) $\underset{k \rightarrow \infty}{\lim} \text{dist}((0,0),\partial E(x^k,x^{k-1})) = 0$.\\
(ii) The sequence $\{x^k\}$ converges to a stationary point of $f$, and $\sum_{k=1}^\infty \|x^k-x^{k-1}\|_H<\infty$.
\end{theorem}

\begin{proof}
We have $\partial  E(x^k, x^{k-1}) = (\nabla f(x^k)+H(x^k-x^{k-1}), - H(x^k - x^{k-1}) )$. Using the definition of $x^k$ in \Eqn{MMextrapolate}, we have
\begin{eqnarray}\label{eq:stationary} 
-\nabla f(y^{k-1}) - H (x^k - y^{k-1}) =0
\end{eqnarray}
By adding \Eqn{stationary} to the left coordinates of $\partial E$, we have
\begin{eqnarray*}
 \partial E(x^k, x^{k-1}) & = & (\nabla f(x^k) - \nabla f(y^{k-1}) -H(x^{k-1}-y^{k-1}), - H(x^k - x^{k-1}) ) 
\end{eqnarray*}
Then we have
\begin{eqnarray*}
\text{dist}((0,0),\partial E(x^k,x^{k-1})) & \le & \|\nabla f(x^k) - \nabla f(y^{k-1})\| + \|H (x^{k-1}-y^{k-1})\| + \| H (x^k - x^{k-1})\|
\end{eqnarray*}
Combining \Lem{normequiv}, \Lem{keylemma1}, and \Lem{keylemma2}, we have
\begin{eqnarray*}
\|\nabla f(x^k) - \nabla f(y^{k-1})\| &\le & \sqrt{\sigma_{\max}(H)} \sqrt{(\nabla f(x^k) - \nabla f(y^{k-1}))^\top H^{-1} (\nabla f(x^k) - \nabla f(y^{k-1}))}\\
& \le & \sqrt{\sigma_{\max}(H)} \sqrt{(\nabla f(x^k) - \nabla f(y^{k-1}))^\top(x^k - y^{k-1})}\\
& \le &\sqrt{\sigma_{\max}(H)} \| x^k - y^{k-1}\|_H\\
& = & \sqrt{\sigma_{\max}(H)}\| x^k - x^{k-1} - \beta_{k-1}(x^{k-1}-x^{k-2})\|_H\\
& \le &\sqrt{\sigma_{\max}(H)}(\|x^k - x^{k-1}\|_H+\|x^{k-1}-x^{k-2}\|_H),
\end{eqnarray*}
\begin{eqnarray*}
\|H (x^{k-1}-y^{k-1})\| &  =  & \|\beta_{k-1}H(x^{k-1}-x^{k-2})\| 
\le \sqrt{\sigma_{\max}(H)} \|x^{k-1}-x^{k-2}\|_H,\\
\|H (x^{k}-x^{k-1})\|  &\le &\sqrt{\sigma_{\max}(H)} \|x^{k}-x^{k-1}\|_H.
\end{eqnarray*}
Adding the above three inequalities allows us to conclude that there is a constant $C$ such that
\begin{eqnarray}\label{eq:distancebound}
\text{dist}((0,0),\partial E(x^k,x^{k-1}))& \le & C(\|x^k - x^{k-1}\|_H+\|x^{k-1}-x^{k-2}\|_H).
\end{eqnarray}
Since $\|x^k - x^{k-1}\|_H \rightarrow 0$, we conclude $\text{dist}((0,0),\partial E(x^k,x^{k-1}))\rightarrow 0$. 

The proof of (ii) follows standard arguments in Theorem 4.2 of \cite{Wen2018APD}. We delay the proof of (ii) to \App{proofglobal}.
\end{proof}

This line of arguments establishes global convergence provided that the step sizes $\{\beta_k\}$ are bounded away from 1. The key is to construct the auxiliary sequence $E_{k}$. We may also study the convergence of iteration \Eqn{MMextrapolate} using tools from convex analysis and fixed-point theory, based on the key observation that the quadratic MM update is $\frac{1}{2}$-averaged in the induced norm $\|\cdot\|_H$. 

\begin{lemma}\label{lem:nonexpansive}
The mapping $T(x) = x - 2H^{-1} \nabla f(x) $ is non-expansive in the $\|\cdot\|_H$ norm. 
\end{lemma}
\begin{proof}
Using \Lem{keylemma1},
\begin{eqnarray*}
\| T(x) - T(y)\|_{H}^2 & = & 
\|x - y\|^2_{H} - 4(x - y)^\top [\nabla f(x)-\nabla f(y)] \\
&  & + \, 4[\nabla f(x) - \nabla f(y)]^\top H^{-1}[\nabla f(x) - \nabla f(y)]\\
& \le & \| x - y\|^2_{H} .
\end{eqnarray*}
\end{proof}

It immediately follows from \Lem{nonexpansive} that the mapping $S(x) = x - H^{-1} \nabla f(x) = \frac{1}{2} x + \frac{1}{2}T(x)$ is $\frac{1}{2}$-averaged in the Hilbert space with inner product $\langle x,y\rangle_H = x^\top H y$. Thus, if in \Eqn{MMextrapolate} we set the inertia parameters $\beta_k=0$,  the original fixed-point iteration reduces to a form of Krasnoselskii-Mann (KM) iteration. There is a significant body of work examining the convergence of Krasnoselskii-Mann iterations with various forms of inertia extrapolation, for instance, see \cite{Alvarez2001AnIP,Moudafi2003ConvergenceOA,Maing2008ConvergenceTF,Lorenz2014AnIF}. We direct readers to \cite{maulen2024inertial} for a more thorough literature review. The recent result in \cite{maulen2024inertial} is particularly useful for our current investigation. An application of the main result of \cite{maulen2024inertial} guarantees the convergence of \Eqn{MMextrapolate} under certain conditions on $\beta_k$. We note that the setting considered in \cite{maulen2024inertial} intends to be as general as possible. Here we state a less general result, which is sufficient for our purpose. 

\begin{lemma}\label{lem:inertialKMconverge}
Let $\mathcal{H}$ be a Hilbert space equipped with norm $\|\cdot\|$ and let $T: \mathcal{H} \rightarrow \mathcal{H}$ be a non-expansive operator with non-empty fixed-point set, i.e., $\text{Fix} \;T\ne \emptyset$. For inertial Krasnoselskii-Mann iterations of the form
\begin{eqnarray}\label{eq:inertialKM}
y^k & = &x^k + \beta_k (x^k - x^{k-1}),\nonumber\\
x^{k+1} & = & (1-\lambda_k)y^k + \lambda_k T y_k,
\end{eqnarray}
suppose that $\beta_k$ and $\lambda_k$ satisfies the following condition:
\begin{eqnarray}\label{eq:conditionbetak}
\limsup_{k\rightarrow\infty} \;[\beta_k(1+\beta_k) + (\lambda_k^{-1}-1)\beta_k(1-\beta_k) - (\lambda_{k-1}^{-1}-1)(1-\beta_{k-1})] <0,
\end{eqnarray}
then $\sum_k \|x^{k+1} - 2x^k +x^{k-1}\|^2$, $\sum_k \|x^{k} - x^{k-1}\|^2$, $\sum_k \|y^k - T y^k\|^2$ are convergent and for each $x^*\in \text{Fix} \;T$, $\underset{k\rightarrow \infty}{\lim} \|x^k-x^*\|$ exists. Moreover, both $x^k$ and $y^k$ converges weakly to a point in $\text{Fix} \;T$. 
\end{lemma}

\begin{theorem}\label{thm:kmconverge}
For the extrapolated quadratic MM algorithm \Eqn{MMextrapolate}, under \As{blanket}, both $x^k$ and $y^k$ converge to an optimal point provided that $\limsup_k \beta_k < \frac{1}{3}$. 
\end{theorem}

\begin{proof}
We have established that $T(x) = x - 2H^{-1} \nabla f(x) $ is non-expansive in the $\|\cdot\|_H$ norm.  \Eqn{MMextrapolate} can thus be viewed as the inertial Krasnoselskii-Mann iterations \Eqn{inertialKM} with relaxation parameter $\lambda_k = \frac{1}{2}$ in the Hibert space defined with inner product $\langle x,y\rangle_H = x^\top H y$. Thus the condition \Eqn{conditionbetak} reduces to
\begin{eqnarray*}
\limsup_{k\rightarrow\infty} \;[ 2\beta_k + \beta_{k-1} -1 ] & < &0,
\end{eqnarray*}
which is satisfied if $\limsup_k \beta_k < \frac{1}{3}$. Then, invoking \Lem{inertialKMconverge}, we conclude that both $x^k$ and $y^k$ converge to a fixed point $x^*$ of $T$, which is also an optimal point since $\nabla f(x^*) = 0$. In finite-dimensional Hilbert spaces, weak and strong convergence are equivalent, so that here we can simply state converge instead of weakly converge. 
\end{proof}

\begin{remark}
 \Thm{kmconverge} disentangles with the Kurdyka–Łojasiewicz (KL) property and guarantees global convergence under only three conditions: the validity of the quadratic majorization \Eqn{quadmajor}, the convexity of the objective function $f$, and the existence of an optimal point. To satisfy the assumptions of \Thm{kmconverge}, one can design a sequence $\{\beta_k\}$ satisfying $\limsup_k \beta_k < \frac{1}{3}$. For example, we may fix $\beta_k=0.333$ after a predetermined number of iterations, or construct a sequence that gradually decays toward 0.333. This strategy ensures convergence even for pathological cases, such as the multinomial regression with full parameterization and 
kernel-regularized learning problems involving a singular kernel matrix $K$, both of which will be detailed in the next section.   
\end{remark}

\section{Application to Kernel Regularized Learning}\label{sec:kernel}
In this section, we introduce our main application of interest: kernel-regularized learning \citep{scholkopf2002learning,mendelson2010regularization}. This framework seeks to estimate a function $f$ by solving the optimization problem
\begin{eqnarray}\label{eq:kernel}
\hat{f} = \underset{f \in \mathcal{H}}{\text{argmin}} \sum_{i=1}^n \ell(b_i,f(a_i)) + \frac{1}{2}\lambda \|f\|_{\mathcal{H}}^2,
\end{eqnarray}
where $(a_i,b_i), i=1,2,…,n$, are i.i.d. observations drawn from an underlying distribution, $l(\cdot,\cdot)$ is a general loss function, $\mathcal{H}$ denotes a reproducing kernel Hilbert space (RKHS) associated with a kernel function $K(\cdot,\cdot)$, and $\lambda>0$ is a regularization parameter balancing the trade-off between data fidelity and model complexity.

Due to the celebrated representer theorem \citep{wahba1990spline,vapnik1999nature}, the solution $\hat{f}$ resides within the span of kernel functions evaluated at the training points. Specifically, we have $\hat{f} = \sum_{i=1}^n \alpha_i K(a_i,\cdot)$ for some coefficient vector $\alpha\in \Real^n$. It can also be shown that $\|f\|_{\mathcal{H}}^2 = \lambda \alpha^\top K\alpha$, where $K = \{K(a_i,a_j)\}_{i,j=1}^n$ is the $n\times n$ kernel matrix. Substituting this representation into the original problem \Eqn{kernel}, we obtain the finite-dimensional optimization problem
\begin{eqnarray}\label{eq:kernelalpha}
\hat{\alpha} = \underset{\alpha \in \Real^n}{\text{argmin}} \sum_{i=1}^n \ell(b_i,(K\alpha)_i) + \frac{1}{2}\lambda \alpha^\top K \alpha.
\end{eqnarray}
When the loss function $\ell$ is the squared loss, the regularized problem \Eqn{kernelalpha} admits a closed-form solution involving the inversion of an $n\times n$ linear system, with computational complexity on the order of $O(n^3)$. For other loss functions, such as quantile, logistic, or multinomial losses, the problem becomes more challenging and typically requires iterative optimization algorithms. For example, for kernel quantile regression (KQR), \citet{takeuchi2006nonparametric} computed the dual problem and solved it using standard quadratic programming techniques, while \citet{li2007quantile} proposed an efficient algorithm that computes the entire regularization path of the KQR. For kernel logistic regression (KLR), \citet{keerthi2005fast} followed the idea of sequential minimal optimization (SMO) to solve the dual problem, which avoids inverting huge Hessian matrices. While \citet{keerthi2005fast} focused on the binary classification case, \citet{zhu2005kernel} extended the method to multi-class classification problems.

The optimization problem \Eqn{kernelalpha} often becomes computationally prohibitive when the sample size $n$ is large. A common strategy to alleviate this burden is to use random projection techniques \citep{alaoui2015fast,yang2017randomized} to restrict the solution of \Eqn{kernelalpha} to an $m$-dimensional subspace, where $m\ll n$ denotes the projection dimension, while aiming to preserve its key statistical properties. If $G\in \Real^{m\times n}$ denotes a random projection matrix, the optimization problem becomes:
\begin{eqnarray}\label{eq:kernelbeta}
\underset{x \in \Real^m}{\text{min}} \quad f(x)=\sum_{i=1}^n \ell(b_i,(KG^\top x)_i) + \frac{1}{2}\lambda x^\top G KG^\top x.
\end{eqnarray}
Popular choices for the sketching matrix $G$ include sub-Gaussian random projections and randomized orthogonal systems (ROS) \citep{yang2017randomized}. The Nyström method \citep{williams2000using,rudi2015less} is a special case of sketching, where $G$ consists of uniformly sampled rows from the $n\times n$ identity matrix. While \citet{alaoui2015fast} showed that sampling rows according to leverage score probabilities can lead to improved statistical performance, uniform sampling often yields comparably good results in practice. Therefore, we adopt uniform sampling in our experiments, though our algorithms are directly applicable to other choices of sketching matrices.

We established the global convergence of \Alg{QMME} in \Sec{algorithm} under the KL property of the potential function $E(x,y)= f(x)+\frac{1}{2}\|x-y\|_H^2$. A natural question is then if $f(x)$ takes the form \Eqn{kernelbeta}, whether the potential function $E(x,y)$ satisfies the KL property. Fortunately, under the fairly mild assumption that $GKG^\top$ is positive definite, the objective \Eqn{kernelbeta} is strongly convex, which implies the strong convexity of $E(x,y)$. Strongly convex functions satisfy the KL property and in fact have a KL exponent \citep{li2018calculus} of $1/2$, which is associated with the local linear convergence of the algorithm. 

\begin{lemma}
The objective \Eqn{kernelbeta} is strongly convex if $\lambda>0$ and $GKG^\top$ is positive definite. Suppose that $\bar{x}$ is the unique minimizer of $f$. Then $E(x,y)$ satisfies the uniform KL property on the singleton set $\{(\bar{x},\bar{x})\}$ with $\phi(s)= \frac{\sqrt{2}}{\sqrt{\mu_E}} s^{1/2}$ where $\mu_E$ is the strong convexity constant of $E(x,y)$.
\end{lemma}
 \begin{proof}
The strong convexity of $\Eqn{kernelbeta}$ follows from the convexity of $ \sum_{i=1}^n \ell(b_i,(KG^\top x)_i)$ and the positive definiteness of $\lambda GKG^\top$. It is easy to see that if $\bar{x}$ is the unique minimizer of $f$, then $(\bar{x},\bar{x})$ is the unique minimizer of $E$.

Assume $f: \mathbb{R}^n \to \mathbb{R}$ is $\mu$-strongly convex ($\mu > 0$) and $H$ is positive definite with minimum singular $\sigma_{\min}(H) > 0$. The Hessian of $E$ is:
\[
\nabla^2 E(x,y) = 
\begin{pmatrix}
\nabla^2 f(x) + H & -H \\
-H & H
\end{pmatrix}.
\]
For any vector $v = (v_x, v_y) \in \mathbb{R}^{2n}$:
\begin{eqnarray*}
v^\top (\nabla^2 E) v 
&=& v_x^\top (\nabla^2 f(x) + H) v_x - 2v_x^\top H v_y + v_y^\top H v_y \\
&=& v_x^\top \nabla^2 f(x) v_x + (v_x - v_y)^\top H (v_x - v_y) \\
&\geq & \mu \|v_x\|^2 + \lambda_{\min}(H) \|v_x - v_y\|^2.
\end{eqnarray*}
This quadratic form is bounded below by:
\begin{eqnarray*}
\mu \|v_x\|^2 + \sigma_{\min}(H) \|v_x - v_y\|^2 
&\geq &\min\left(\mu, \sigma_{\min}(H)\right) \left( \|v_x\|^2 + \|v_x - v_y\|^2 \right) \\
&\geq &\min\left(\mu, \sigma_{\min}(H)\right) \cdot \tfrac{1}{4} (\|v_x\|^2 + \|v_y\|^2),
\end{eqnarray*}
where the last inequality follows from:
\begin{eqnarray*}
\|v_x\|^2 + \|v_x - v_y\|^2 
&= &2\|v_x\|^2 - 2\langle v_x, v_y\rangle + \|v_y\|^2 \\
&\geq& 2\|v_x\|^2 - 2\|v_x\|\|v_y\| + \|v_y\|^2.
\end{eqnarray*}
To show $2\|v_x\|^2 - 2\|v_x\|\|v_y\| + \|v_y\|^2 \geq \frac{1}{4} (\|v_x\|^2 + \|v_y\|^2)$, set $a = \|v_x\|^2$, $b = \|v_y\|^2$, and $c = \|v_x\|\|v_y\|$. Then the equality is written as
\begin{eqnarray*}
2a - 2c + b &\geq &\tfrac{1}{4}(a + b),
\end{eqnarray*}
which further simplifies to:
\begin{eqnarray*}
\tfrac{7}{4}a - 2c + \tfrac{3}{4}b &\geq& 0.
\end{eqnarray*}
Substituting $c = \sqrt{a}\sqrt{b}$ and letting $t = \sqrt{a/b}$:
\begin{eqnarray*}
\tfrac{7}{4}t^2 - 2t + \tfrac{3}{4} & \geq & 0.
\end{eqnarray*}
The discriminant is $(-2)^2 - 4 \cdot \tfrac{7}{4} \cdot \tfrac{3}{4} = 4 - \tfrac{21}{4} = -\tfrac{5}{4} < 0$, and since the quadratic in $t$ has positive leading coefficient, it is always positive. Thus $\nabla^2 E(x,y) \succeq \mu_E I_{2n}$ for $\mu_E = \frac{1}{4} \min\left(\mu, \sigma_{\min}(H)\right) > 0$, proving $E$ is $\mu_E$-strongly convex.

Define function \(\phi(s) = \frac{\sqrt{2}}{\sqrt{\mu_E}} s^{1/2}\) and use $z\in \Real^{2n}$ to denote $(x,y)$. For any \(z \neq \bar{z} =(\bar{x},\bar{x})\), set \(d = E(z) - E(\bar{z}) > 0\) and \(a = \|z - \bar{z}\|\). By strong convexity we have $d \geq \frac{\mu_E}{2}a^2$ and
\[
d \leq \langle \nabla E(z), z - \bar{z} \rangle - \frac{\mu_E}{2}a^2 \leq \|\nabla E(z)\| a - \frac{\mu_E}{2}a^2.
\]
Solving for the gradient norm gives:
\[
\|\nabla E(z)\| \geq \frac{d}{a} + \frac{\mu_E}{2}a.
\]
Substituting the optimal bound \(a = \sqrt{2d/\mu_E}\) (which minimizes the right-hand side) yields:
\[
\|\nabla E(z)\| \geq \sqrt{2\mu_E d}.
\]
Finally, we verify the KL condition:
\[
\varphi'(d) \|\nabla E(x)\| = \left( \frac{1}{\sqrt{2\mu_E}} d^{-1/2} \right) \|\nabla E(x)\| \geq \frac{1}{\sqrt{2\mu_E}} d^{-1/2} \cdot \sqrt{2\mu_E} d^{1/2} = 1,
\]
which holds globally for all $z\ne \bar{z}$.
\end{proof}
When $E(x,y)$ is strongly convex and consequently has an KL exponent of $\frac{1}{2}$, we may establish the local linear convergence of \Alg{QMME}.
\begin{theorem}\label{thm:locallinear}
If $\lambda>0$ and $GKG^\top$ is positive definite, then \Alg{QMME} applied to \Eqn{kernelbeta} generates a sequence $\{x^k\}$ that converges locally linearly. In other words, there exists $c_1>0$, $k_1>0$ and $0<\eta<1$ such that $\|x^k - \bar{x}\|<c_1\eta^k$ for $k>k_1$, where $\bar{x}$ in the unique minimizer of $f$.
\end{theorem}
\begin{proof}
See \App{locallinear}.
\end{proof}

We now turn to three specific kernel-based learning problems: kernel (smoothed) quantile regression, kernel logistic regression, and kernel multinomial regression. For each, we present the corresponding optimization formulation, derive a quadratic majorant, and provide the normal equation associated with minimizing this surrogate objective. We also highlight how key matrix decompositions can be efficiently reused across iterations, and in some cases, across different values of the regularization parameter $\lambda$.

\subsection{Kernel Smoothed Quantile Regression}
Quantile regression is a powerful and time-tested tool for modeling the entire conditional distribution of a response variable, offering insights beyond the conditional mean captured by traditional least squares regression \citep{koenker1978regression, koenker2005quantile}. Compared to least-squares regression, quantile regression is more robust to outliers and provides a more comprehensive understanding of the relationship between input and output variables.
The core optimization problem in quantile regression involves minimizing the piecewise linear quantile loss function $\rho_\tau(r) = (\tau-\frac{1}{2})r+\frac{1}{2}|r|$, also known as the check loss function, which is non-smooth, non-differentiable, and lacks strong convexity at $r=0$. These properties make the use of standard gradient-based optimization methods inefficient or inapplicable. Recent developments in the quantile regression literature have introduced smooth alternatives based on convolution smoothing \citep{fernandes2021smoothing,he2023smoothed}, where the loss is given by 
\begin{eqnarray*}
\ell_{h,\tau}(u) = (\rho_{\tau}\ *\ K_h)(u)  =  \int_{-\infty}^\infty \rho_\tau(v) K_h(v-u)dv.
\end{eqnarray*}
Here, $*$ denotes the convolution operator, $K_h$ is a smoothing kernel, i.e., a probability density function with unit integral, and $h>0$ is a bandwidth parameter that controls the degree of smoothing. In this paper, we focus on the Gaussian kernel $K_h(u) = \frac{1}{\sqrt{2\pi}h}\exp(-\frac{u^2}{2h^2})$, and thus the smoothed quantile loss is 
\begin{eqnarray*}
\ell_{h,\tau}(u) & = &\frac{h}{\sqrt{2\pi}}\exp(-\frac{u^2}{2h^2})+\frac{u}{2}\left\{1-2\Phi(-\frac{u}{h})\right\}+(\tau-\frac{1}{2})u,
\end{eqnarray*}
where $\Phi(\cdot)$ denotes the CDF of standard normal distribution. Smaller values of $h$ yield tighter approximations to the original non-smooth check loss.

Using the smoothed quantile loss in the data fidelity term, we now minimize the following kernel smoothed quantile regression objective function:
\begin{eqnarray*}\label{eq:quantileobj}
f(x) = \sum_{i=1}^n \ell_{h,\tau}( b_i - (KG^\top x)_i ) + \frac{1}{2}\lambda x^\top G KG^\top x.
\end{eqnarray*}
The gradient and Hessian of $f$ are:
\begin{eqnarray}\label{eq:gradientHessian}
\nonumber \nabla f(x) & = & - GK \psi + \lambda GKG^\top x,\\
\nabla^2 f(x) & = & G K W KG^\top + \lambda GKG^\top,
\end{eqnarray}
where $\psi_i = \ell'_{h,\tau} (r_i), \;i = 1,2,\dots,n$ and $r_i = b_i - (KG^\top x)_i$. $W$ is an $m\times m$ diagonal matrix with diagonal entries $W_{ii} = \ell''_{h,\tau} (r_i)$. Leveraging the fact that $\ell''(u) = K_h(u)\le \frac{1}{\sqrt{2\pi}h}$, we can upper bound the Hessian by
\begin{eqnarray*}\label{eq:quantileH}
G K W KG^\top + \lambda GKG^\top \preceq H =\frac{1}{\sqrt{2\pi}h}GK^2G^\top + \lambda GKG^\top+\delta I,
\end{eqnarray*}
where the damping term $\delta I$ is introduced to improve numerical stability. 

To minimize the quadratic surrogate, we compute $H^{-1}\nabla f(y_k)$ at each iteration. This is performed efficiently using Cholesky decomposition: we factorize the matrix  $H=LL^\top$, where $L$ is a lower-triangular matrix, and solve the linear system via forward and backward substitutions, i.e., $H^{-1}\nabla f(y_k)=(L^\top)^{-1}L^{-1}\nabla f(y_k)$. Crucially, $H$ itself and its Cholesky decomposition are computed only once during initialization and subsequently reused across all iterations, significantly reducing the computational cost of each update.
\subsection{Kernel Logistic Regression}
Kernel logistic regression (KLR) \citep{canu2006kernel,jaakkola1999probabilistic,keerthi2005fast} is a widely used method for binary classification tasks. Compared to another popular kernel-based classification method, support vector machines (SVM), the advantages of KLR include its probabilistic interpretation and the ability to generalize to multi-class classification problems that we will discuss in the next subsection. 

The objective of kernel logistic regression is written as
\begin{eqnarray*}\label{eq:logisticobj}
f(x) = -\sum_{i=1}^n [b_i\log(p_i)+ (1-b_i)\log(1-p_i)]+ \frac{1}{2}\lambda x^\top G KG^\top x,
\end{eqnarray*}
where $p_i = 1/[1+\exp(-(KG^\top x)_i)]$ and $b_i\in\{0,1\}$. The gradient and Hessian in this case take the same format as \Eqn{gradientHessian}, but in this case $\psi_i = b_i - p_i$ while $W_{ii} = p_i(1-p_i)$. Using the fact that $p_i(1-p_i)\le \frac{1}{4}$, we obtain the upper bound
\begin{eqnarray}\label{eq:logisticH}
G K W KG^\top + \lambda GKG^\top \preceq H =\frac{1}{4}GK^2G^\top + \lambda GKG^\top+\delta I.
\end{eqnarray}
As with kernel quantile regression, the Cholesky decomposition of $H$ in \Eqn{logisticH} can be computed once and reused throughout the MM iterations.

\subsection{Kernel Multinomial Regression}
In multi-class classification problems, kernel methods are typically implemented using one-vs-one or one-vs-rest strategies, as seen in kernel logistic regression and kernel support vector machines. \citet{zhu2005kernel} extended kernel methods to directly handle the multi-class setting via a kernel multinomial formulation. However, this approach has received relatively limited attention, likely due to the greater optimization challenges it entails. Indeed, compared with kernel logistic regression and support vector machines, the number of parameters is further multiplied by the number of classes. In this subsection, we present new algorithms that cleverly exploit the special structure of multinomial regression to substantially reduce computational costs.

In a standard multinomial regression model, if $b_i\in\{1,2,\dots,q\}$ indicates the category of sample $i$, using the $q$-th category as the reference category, the probability $p_{ij}$ that $b_i=j$ is computed by
\begin{eqnarray*}
p_{ij}(X) & = & \begin{cases} \frac{e^{\eta_{ij}}}{1+\sum_{k=1}^{q-1} e^{\eta_{ik}}} & 1 \le j < q \\
\frac{1}{1+\sum_{k=1}^{q-1} e^{\eta_{ik}}} & j=q ,
\end{cases}
\end{eqnarray*}
where $\eta_{ij} =(KG^\top X_j)_i$ and $X_j$ is the $j$-th column of the $m\times (q-1)$ coefficient matrix $X$. The kernel multinomial regression objective function is
\begin{eqnarray}\label{eq:multinomialobj}
f(X) = -\sum_{i=1}^n \sum_{j=1}^q I(b_i=j)\log(p_{ij})+ \frac{1}{2}\lambda\; \text{tr}(X^\top G KG^\top X),
\end{eqnarray}
where $\text{tr}(\cdot)$ denotes the trace of a matrix. To simplify notation, we define  $B\in \Real^{n \times (q-1)}$ and $P\in \Real^{n\times (q-1)}$, where each row of $B$ is a binary vector such that $b_{ij}=1$ if $b_i=j$, and each row of $P$ contains the corresponding probabilities $p_{ij}$ for $j=1,\dots,q-1$. By vectorizing $X$ as $\text{vec}(X)$, the gradient and Hessian of $f$ can be expressed by
\begin{eqnarray}\label{eq:gradientHessianmultinomial}
\nonumber \nabla f(\text{vec}(X)) & = & - \text{vec}(GK(B-P))+\lambda\text{vec}(GKG^\top X),\\
\nabla^2 f(\text{vec}(X)) & = & \sum_{i=1}^n (\Lambda_{p_i}-p_ip_i^\top)\Kron (KG^\top)_i(KG^\top)_i^\top+ \lambda I_{q-1}\otimes GKG^\top,
\end{eqnarray}
where $p_i$ is the $i$-th row of $P$, $(KG^\top)_i$ is the $i$-th row of $KG^\top$, and $\Lambda_{p_i}$ is a diagonal matrix with diagonal entries filled by $p_i$. To make sense of the matrix calculus involved in deriving the gradient and Hessian, we direct readers to \cite{bohning1992multinomial}. 

The above model adopts the standard parameterization of multinomial regression, where the coefficient matrix has $q-1$ columns when there are $q$ categories in total. This is achieved by selecting one of the categories as the reference category. Here we have followed the convention of setting the last category as the reference category. However, the choice of reference category can have a non-negligible impact on the performance of the multinomial regression model \citep{fu2023simplex}. An alternative parameterization is to treat all categories as equal, with the coefficient matrix having $q$ columns. We refer to this parameterization as the ``full parameterization". In the fully parameterized multinomial regression model, the probability $p_{ij}$ is computed as
\begin{eqnarray*}
p_{ij}(X) & = & \frac{e^{\eta_{ij}}}{\sum_{k=1}^{q} e^{\eta_{ik}}},\quad 1 \le j \le q. \\
\end{eqnarray*}
The objective, gradient, and Hessian take the same format as \Eqn{multinomialobj} and \Eqn{gradientHessianmultinomial} with the following distinctions. First, the coefficient matrix $X$, the response matrix $B$, and the probability matrix $P$ has $q$ columns instead of $q-1$ columns, with each column representing one category. In the fully parameterized model, each row of $P$ sums to 1, which is not the case for the standard model. Second, we need to replace $I_{q-1}$ with $I_q$ in \Eqn{gradientHessianmultinomial}. Notice that since $p_i$ sums to 1, $\Lambda_{p_i}-p_ip_i^\top$ is in fact a singular matrix, as $1_q$ is in its null space.  Thus, should the penalty parameter $\lambda$ be 0, the Hessian of the objective is in fact singular.  The objective is pathological in the sense that it remains constant if we add multiples of $1_q$ to any row of $X$. However, we find that as long as the quadratic upper bound is positive definite, the QMME framework is still effective in solving the optimization problem.

We now derive upper bounds for the Hessian using both the standard parameterization and the full parameterization. Using an inductive argument, \cite{bohning1988monotonicity} showed that for any probability vector $p\in \Real^q$ that lies on the $q$-dimensional simplex, that is, $p_i$ is nonnegative and $\sum_{i=1}^q p_i=1$, we have
\begin{eqnarray}\label{eq:simplexupperbound}
\Lambda_{p_i}-p_ip_i^\top & \preceq & \frac{1}{2} (I_{q} - \frac{1}{q}1_{q} 1_{q}^\top).
\end{eqnarray}
In the above inequality, both sides are singular since $1_q$ are in their null spaces. It is not hard to see that 
\begin{eqnarray*}
\frac{1}{2} (I_{q} - \frac{1}{q}1_{q} 1_{q}^\top) & \preceq & E_1 = \frac{1}{2} (I_{q} - \frac{1}{q+1}1_{q} 1_{q}^\top). 
\end{eqnarray*}
Thus, for the full parameterization model, we have
\begin{eqnarray*}\label{eq:Efull}
\Lambda_{p_i}-p_ip_i^\top & \preceq & E_1 =\frac{1}{2} (I_{q} - \frac{1}{q+1}1_{q} 1_{q}^\top).
\end{eqnarray*}
Now, let us consider arbitrary vectors of the format $a = (v^\top,0)^\top$ where $v\in \Real^{q-1}$. Due to \Eqn{simplexupperbound}, we have
\begin{eqnarray*}
a^\top (\Lambda_{p_i}-p_ip_i) a &\le& a^\top (I_{q} - \frac{1}{q}1_{q} 1_{q}^\top) a,
\end{eqnarray*}
which reduces to
\begin{eqnarray*}
v^\top (\Lambda_{s_i}-s_is_i) v &\le& v^\top (I_{q-1} - \frac{1}{q}1_{q-1} 1_{q-1}^\top) v,
\end{eqnarray*}
where $s_i\in \Real^{q-1}$ are the first $q-1$ elements of $p_i$. Thus we have
\begin{eqnarray}\label{eq:Ereduced}
(\Lambda_{s_i}-s_is_i)  &\preceq & E_2= (I_{q-1} - \frac{1}{q}1_{q-1} 1_{q-1}^\top) ,
\end{eqnarray}
which is the upper bound given in \cite{bohning1992multinomial}.

Leveraging \Eqn{Ereduced}, in the standarad parameterization, the Hessian of $f$ has upper bound
\begin{eqnarray*}
\nabla^2 f(\text{vec}(X)) & \preceq & H = E_2 \otimes GK^2G^\top + \lambda I_{q-1}\otimes GKG^\top + \lambda \delta I_{(q-1)\times m}.
\end{eqnarray*}
Similarly, in the full parameterization, 
 the Hessian of $f$ has upper bound
\begin{eqnarray*}
\nabla^2 f(\text{vec}(X)) & \preceq & H = E_1 \otimes GK^2G^\top + \lambda I_{q}\otimes GKG^\top + \lambda \delta I_{q\times m}.
\end{eqnarray*}
We subsequently use \( E \) to denote either \( E_1 \) or \( E_2 \). Incidentally, both \( E_1^{-1} \) and \( E_2^{-1} \) take similar forms:
\[
E_1^{-1} = 2(I_q + 1_q 1_q^\top), \quad \text{and} \quad E_2^{-1} = 2(I_{q-1} + 1_{q-1} 1_{q-1}^\top).
\]
We now describe the Sylvester equation technique for minimizing the quadratic majorant using the standard parameterization. The corresponding workflow for the full parameterization is nearly identical. Under the standard parameterization, the normal equation for minimizing the quadratic majorant in \Eqn{overallmajor} is given by:
\begin{eqnarray*}
(E \otimes GK^2G^\top + \lambda I_{q-1}\otimes GKG^\top + \lambda\delta I_{(q-1)\times m}) \text{vec}(\Delta) = \text{vec}(C^k),
\end{eqnarray*}
where $\text{vec}(C^k) = \nabla f(\text{vec}(Y^k))$ and $Y^k\in \Real^{m\times (q-1)}$ is the extrapolated MM anchor at iteration $k$. Using the property of the Kronecker product, $(A\otimes B) \text{vec}(C)= \text{vec} (CBA)$, the linear system can be equivalently rewritten in matrix form
\begin{eqnarray*}
GK^2G^\top \Delta E + \lambda( GKG^\top+\delta I_{m}) \Delta & = & C^k.
\end{eqnarray*}
Left-multiplying the above equation by $(GKG^\top+\delta I_{m})^{-1}$ and right-multiplying it by $E^{-1}$, we obtain
\begin{eqnarray}\label{eq:sylvester}
(GKG^\top+\delta I_{m})^{-1}GK^2G^\top \Delta + \Delta (\lambda E^{-1})& = &( GKG^\top+\delta I_{m})^{-1}C^k E^{-1}.
\end{eqnarray}
This is a classic Sylvester equation of the form $AX+XB=C$ \citep{bartels1972algorithm}. The Sylvester equation admits a unique solution provided that the spectra of $A$ and $-B$ are disjoint (e.g., their singular values do not overlap). Although Sylvester equations have been extensively studied in control theory, their prominent application in statistics is more recent \citep{xu2022proximal,heng2025tactics}. The general approach to solving a Sylvester equation begins with finding orthogonal matrices $U$ and $V$ such that $A=UA'U^\top$ and $B=VB'V^\top$ where $A'$ is lower-triangular and $B'$ is upper-triangular. By left multiplying the Sylvester equation with $U^\top$ and right multiplying with $V$, it transforms into $A'U^\top \Delta V+U^\top\Delta VB' = U^\top C V$. Treating $U^\top \Delta V$ as a new variable $Z$, the equation reduces to $A'Z+ZB'=U^\top C V$, which can be solved analytically using the formula in equation (3) of \cite{bartels1972algorithm}, exploiting the lower-triangular and upper triangular-structure of $A'$ and $B'$. The solution $X$ to the original Sylvester equation is then recovered by $X=UZV^\top$.

\begin{algorithm}[t]
\caption{Solving Sylvester Equation \Eqn{sylvester}}\label{alg:sylvester}
\begin{algorithmic}[1]
\STATE Compute $(GKG^\top+\delta I_{m})^{-1}GK^2G^\top$.
\STATE Compute lower-Schur decomposition $(GKG^\top+\delta I_{m})^{-1}GK^2G^\top = U A' U^\top$.
\STATE Compute spectral decomposition $\lambda E^{-1} = V A' V^\top$.
\STATE Compute Cholesky decomposition $GK^2G^\top +\delta I_{m} = LL^\top$.
\STATE Compute right hand side of \Eqn{sylvester} as $C=(L^\top)^{-1}L^{-1}C^k E^{-1}$.
\STATE Solve $A'Z+ZB'=U^\top C V$ via equation (3) of \cite{bartels1972algorithm}.
\STATE Recover $\Delta$ by $\Delta=UZV^\top$.
\end{algorithmic}
\end{algorithm}

A complete recipe for solving \Eqn{sylvester} is provided in in \Alg{sylvester}. We triangularize $(GKG^\top+\delta I_{m})^{-1}GK^2G^\top$ and $\lambda E^{-1}$ using two different matrix decompositions since the former is generally asymmetric, while the latter is symmetric and positive definite. Line 1,2,3,4 of \Alg{sylvester} only needs to be computed once across all iterations and all values of $\lambda$ (notice that the spectral decomposition of $\lambda E^{-1}$ differs from the the spectral decomposition of $E^{-1}$ only by a constant factor of $\lambda$). We still need to compute lines 5,6,7 at each iteration.

\section{Numerical Experiments}\label{sec:ne}
In this section, we use simulated data to assess the computational efficiency and statistical accuracy of our MM algorithms applied to the three kernel regularized learning problems introduced earlier. We generate an $n\times d$ design matrix $A$ with covariance matrix $\Sigma$ where $\Sigma_{ij} = 0.5^{|i-j|}$ and $d=50$.  For quantile regression, the response $b_i$ is generated as
\begin{eqnarray*}
\eta_i  & = & -4+\text{sin}(a_{i1}) +a_{i2}a_{i3} + a_{i4}^3 - |a_{i5}| + \frac{1}{10}\|(a_{i6},\dots,a_{id})\|^2, \quad   
b_i  = \eta_i + \epsilon_i,
\end{eqnarray*}
where $\epsilon_i$ follows a student's $t$ distribution with 1.5 degrees of freedom. For logistic regression, $\eta_i$ is generated in the same manner, but the response is generated as $b_i \sim \text{Bin}(1,p_i)$, where $p_i = 1/[1+\exp(-\eta_i)]$. For multinomial regression, the data generation protocol is
\begin{eqnarray*}
\eta_{1i}  & = & -4+\text{sin}(a_{i1}) +a_{i2}a_{i3} + a_{i4}^3 - |a_{i5}| + \frac{1}{10}\|(a_{i6},\dots,a_{id})\|^2, \\
\eta_{2i}   & = & 4+\text{cos}(a_{i1}) +a_{i2}a_{i4} + a_{i5}^3 - |a_{i3}| - \frac{1}{10}\|(a_{i6},\dots,a_{id})\|^2, \\
p_{1i} & = & \exp(\eta_{1i})/[\exp(\eta_{1i})+\exp(\eta_{2i})+1], \\ p_{2i} & = &  \exp(\eta_{2i})/[\exp(\eta_{1i})+\exp(\eta_{2i})+1],\\
b_i & \sim & \text{Multinomial}[\{1,2,3\},(p_{1i},p_{2i},1-p_{1i}-p_{2i})],
\end{eqnarray*}
where $b_i  \sim  \text{Multinomial}[\{1,2,3\},(p_{1i},p_{2i},1-p_{1i}-p_{2i})]$ suggests that $b_i$ has probability $p_{1i}$, $p_{2i}$, $1-p_{1i}-p_{2i}$ to take the values 1, 2, 3, respectively.

We employ the Radial Basis Function (RBF) kernel $K(a,a') = \exp \{-\|a-a'\|^2/(2\sigma^2)\}$. The bandwidth parameter $\sigma$ is hand-tuned and will be specified prior to each experiment. All experiments are implemented in Julia 1.9.3, and conducted on a Mac with an M4 chip and 32GB of RAM. For quantile regression, the smoothing parameter is set to $h = 0.25$, and the quantile is set to $\tau = 0.5$. In kernel quantile regression and logistic regression, the damping parameter is set to $\delta = 10^{-9}$. For multinomial regression, the damping matrix is configured as $\lambda\delta I$, with $\delta = 10^{-4}$. The experiments presented in this paper are fully reproducible, with code and vignettes available at the repository \url{https://github.com/qhengncsu/QMME.jl}.

\subsection{Convergence Speed on One Instance}\label{sec:oneinsta}

We begin our experiments by comparing the convergence speed of our quadratic MM algorithms against several popular first- and second-order optimization methods on a single randomly generated instance. The sample size is fixed at $n = 2^{14}$, and the sketching matrix $G$ is set to have $m = 2^{11}$ rows. The penalty parameter is uniformly set to $\lambda = 10^{-4}$ across all methods. Each algorithm is terminated either when the gradient norm satisfies $\|\nabla f(x^k)\| < 10^{-4}$ or after a maximum of 1000 iterations. The kernel bandwidth is set to $\sigma=20$ throughout this subsection.

The competing methods include:
\begin{itemize}
\item The Fast Iterative Shrinkage-Thresholding Algorithm (FISTA) \citep{Beck2009AFI}, with the Lipschitz constant $L$ computed individually for each problem instance. As we are dealing with smooth objectives, the proximal step in FISTA is omitted.
\item Adaptive Gradient Descent (AdaGD), implemented according to Algorithm 1 in \cite{malitsky2020adaptive}, initialized with a step size of $10^{-7}$.
\item Standard Newton’s method, incorporating step size halving to ensure descent of the objective. As in our MM framework, a small ridge term ($10^{-9}I$) is added to the Hessian matrices for numerical stability.
\end{itemize}

We also experimented with the Barzilai–Borwein (BB) method \citep{barzilai1988two}, which has shown strong empirical performance in low-dimensional smoothed quantile regression problems \citep{he2023smoothed}. However, in our kernel regression setting, BB consistently diverged—likely due to the ill-conditioning of the kernel matrix $K$ and the increased complexity of kernel quantile regression. As a result of its erratic and divergent behavior, BB is excluded from the visual comparisons for clarity. For kernel multinomial regression, we restrict our attention to the standard parameterization in this subsection. We note that for the full parameterization, the Hessian matrix often becomes nearly singular when $\lambda$ is small, rendering Newton’s method inapplicable due to numerical instability.

\begin{figure}[t]
  \centering
  \includegraphics[width=\textwidth]{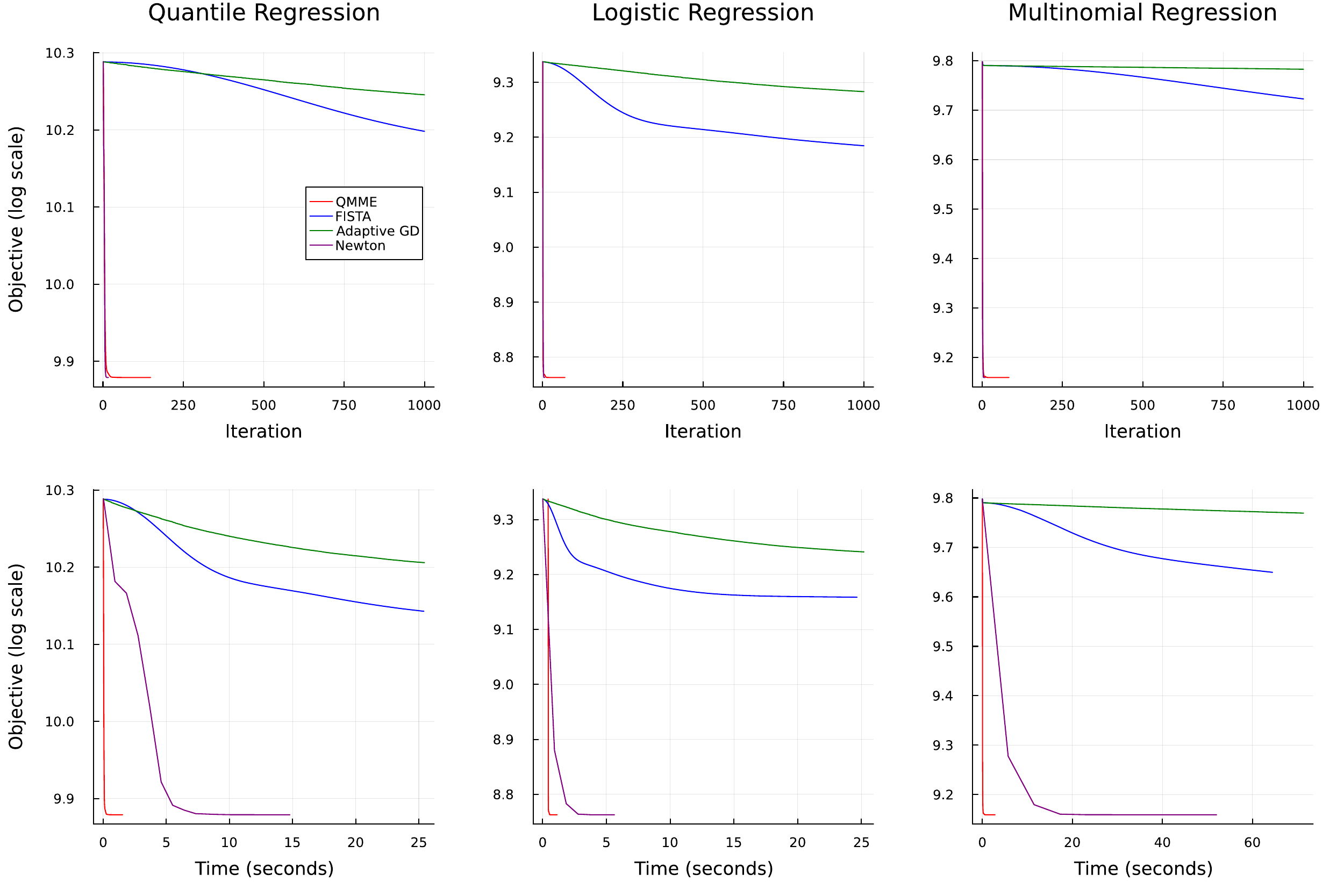}
  \caption{Log scale objective trajectories of different algorithms on one instance of the three problems considered. The first row shows the objective versus the number of iterations, while the second row shows the objective versus the computation time.}
  \label{fig:obj}
\end{figure}

\begin{figure}[t]
  \centering
  \includegraphics[width=\textwidth]{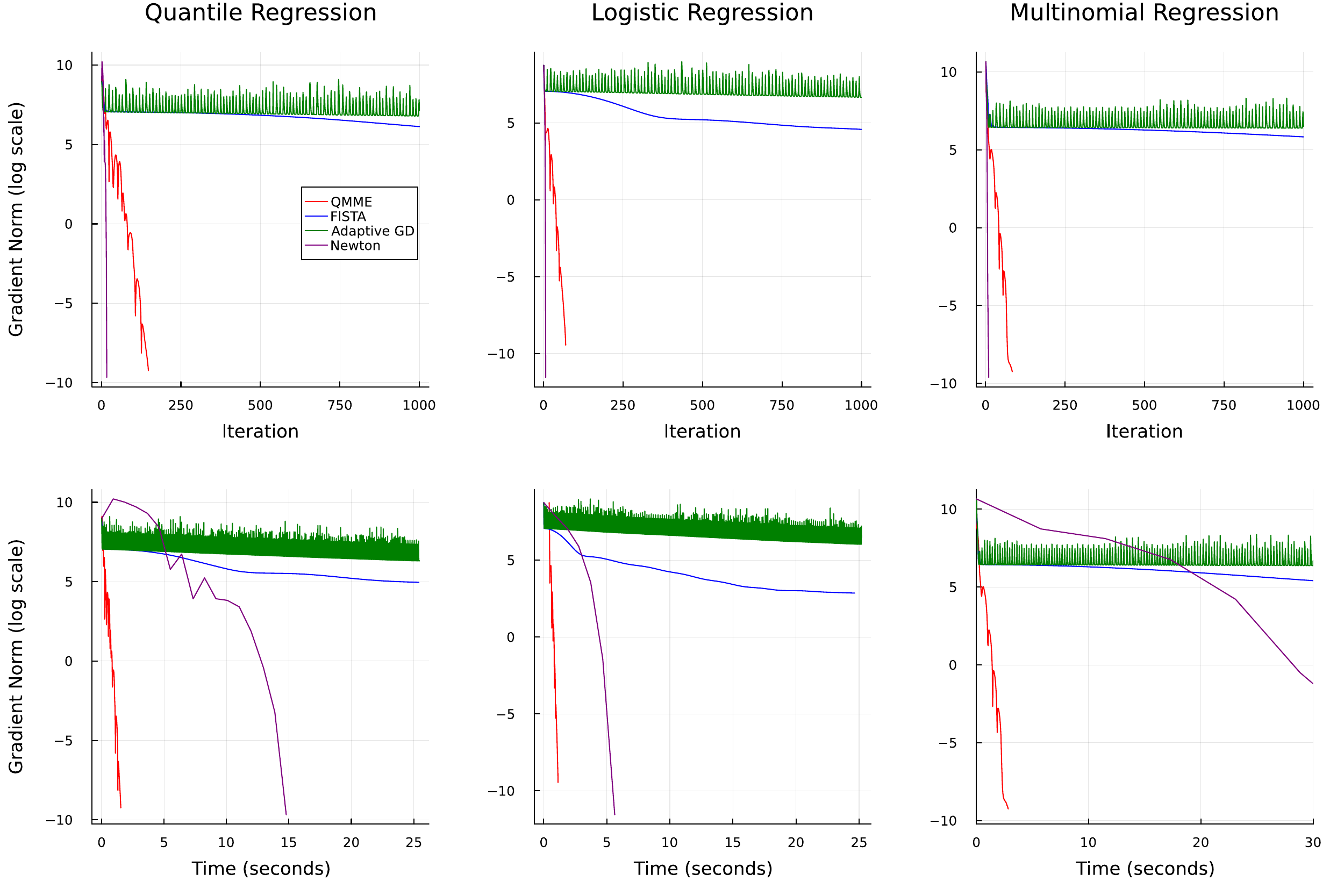}
  \caption{Log scale gradient norm trajectories of different algorithms.}
  \label{fig:grad}
\end{figure}

Our quadratic MM algorithms converge more than an order of magnitude faster than FISTA, both in terms of iteration count and total computation time. Somewhat unexpectedly, the adaptive gradient method of \cite{malitsky2020adaptive} performs significantly worse than FISTA, exhibiting slow convergence and pronounced oscillations in the gradient norm. The poor performance of both AdaGD and BB suggests that these kernel learning problems may pose intrinsic challenges for step-size-adaptive first-order methods. Overall, first-order algorithms fail to achieve the desired accuracy within a reasonable number of iterations. In our experiments, only methods that incorporate curvature information, namely QMME and Newton Raphson, are able to satisfy the convergence criterion $\|\nabla f(\beta_k)\| < 10^{-4}$ within 1000 iterations.

This observation is consistent with the findings of \cite{liu2025kernel}, who reported that a second-order Iteratively Reweighted Least Squares (IRLS) method converged substantially faster than gradient descent for kernel logistic regression. Although the specifics of their IRLS implementation were not disclosed, we believe it closely resembles the Newton method evaluated here. Newton’s method typically converges in fewer than 10 iterations; however, the high per-iteration cost leads to total runtimes that are several times longer than those of our quadratic MM algorithm. These results suggest that, at the appropriate scale, the QMME framework achieves a more favorable trade-off between iteration count and per-iteration computational cost.

\subsection{Computing the Entire Solution Path}
\Sec{oneinsta} focused on visualizing the objective and gradient norm trajectories for a fixed value of $\lambda$. In practice, however, selecting the tuning parameter typically requires computing the entire solution path. This is commonly accomplished by solving the optimization problem over a sequence of decreasing $\lambda$ values, where the solution from the previous step serves as a warm start for the next. While random projection techniques can significantly reduce the computational cost of kernel methods, achieving accurate approximations often necessitates larger projection dimensions, particularly when the number of features $d$ is large. In general, the statistical accuracy of kernel learning improves monotonically with projection dimension. As such, more efficient algorithms for solving \Eqn{kernelbeta} complement existing kernel approximation techniques by enabling scalable computation at larger values of $m$. In this subsection, we present carefully designed experiments to illustrate these points. Given the limitations of first-order methods demonstrated in the previous section, we focus on the QMME and Newton–Raphson algorithms for the remaining experiments.

We adopt the same simulation setup as in \Sec{oneinsta}, but vary the projection dimension by considering $m \in \{2^3, 2^4, \dots, 2^{12}\}$. For kernel multinomial regression, due to the prohibitive computational cost of Newton–Raphson at $m = 2^{12}$, we instead consider $m \in {2^2, 2^3, \dots, 2^{11}}$. An independent validation set of size $n/4 = 4096$ is generated for model selection. For each random replicate and each value of $m$, we solve the optimization problem along a sequence of 30 exponentially decreasing $\lambda$ values, ranging from $10^1$ to $10^{-5}$. We record both the total time required to compute the full solution path and the best validation accuracy achieved across all $\lambda$ values. To ensure a fair comparison, we adopt the same convergence criterion for both QMME and Newton–Raphson: $\|\nabla f(x^k)\| < 10^{-4}$.

Accuracy is evaluated using problem-specific metrics. For quantile regression, we report the mean absolute deviation between $f(a_i)$ and $\hat{f}(a_i)$ on the validation set, where lower values indicate better predictive performance. For logistic and multinomial regression, accuracy is assessed using the validation log-likelihood, with higher values indicating better fit. In the case of kernel multinomial regression, we consider both the standard and full parameterizations, which correspond to two closely related but slightly different optimization objectives. Consequently, minor discrepancies in the resulting validation log-likelihoods are expected.

\begin{figure}[t]
  \centering
  \includegraphics[width=\textwidth]{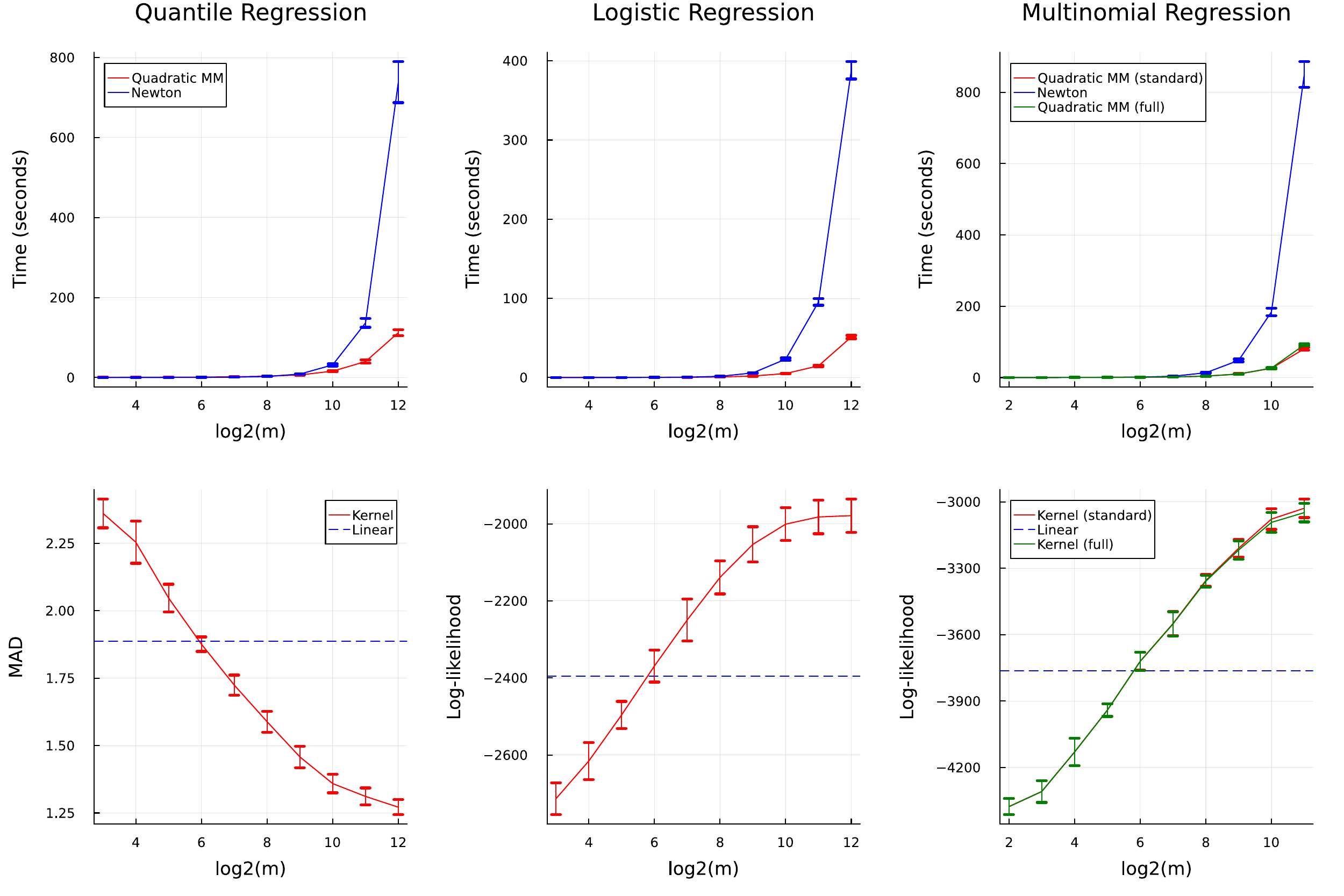}
  \caption{Total time to compute entire solution path and best accuracy on the validation set for different algorithms. The error bar shows one standard error using bandwidth $\sigma=15$. All data points are averaged over 20 random replicates. The blue horizontal line indicates the accuracy of a linear model.}
  \label{fig:path1}
\end{figure}
\begin{figure}[t]
  \centering
  \includegraphics[width=\textwidth]{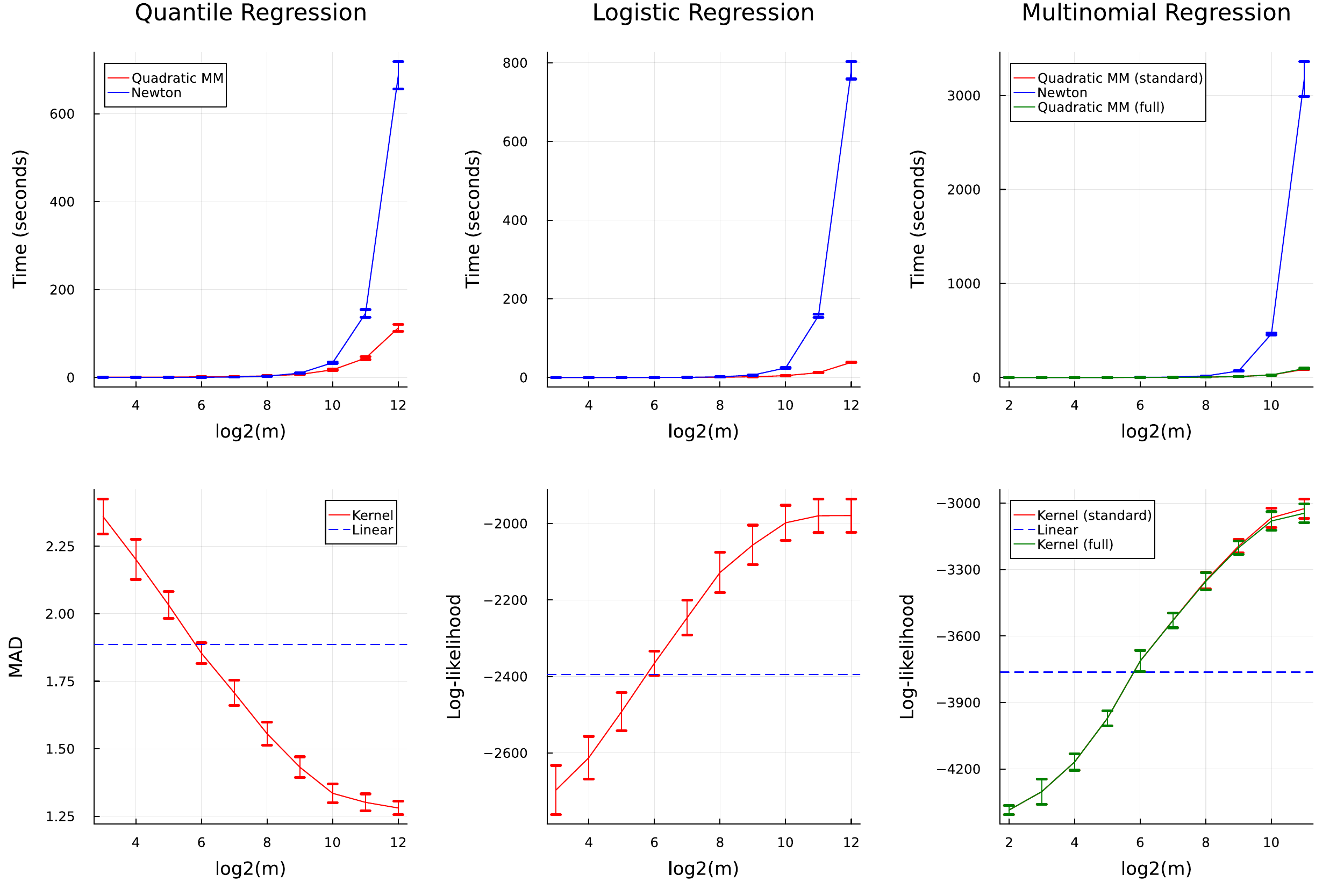}
  \caption{Total time to compute entire solution path and best accuracy on the validation set using bandwidth $\sigma=30$. All data points are averaged over 20 random replicates.}
  \label{fig:path2}
\end{figure}

\Fig{path1} and \Fig{path2} present experimental results for kernel bandwidths $\sigma = 15$ and $\sigma = 30$, respectively. For small values of $m$, the Newton-Raphson method exhibits superior efficiency due to its rapid convergence. However, as $m$ increases, the computational benefits of reusing matrix decompositions in QMME become increasingly pronounced. In general, QMME begins to outperform Newton-Raphson when $m \geq 2^{10}$, with the performance gap widening as $m$ continues to grow. In the context of kernel multinomial regression, QMME significantly outpaces Newton-Raphson for large $m$, thanks to the use of the Sylvester equation technique that exploits the special structure of the Hessian.

The second rows of \Fig{path1} and \Fig{path2} illustrate that the best achievable accuracy of kernel methods tends to improve with increasing projection dimension $m$. While this improvement may plateau earlier for low-dimensional data, it remains substantial for datasets with large ambient dimension $d$. The results for $\sigma = 15$ and $\sigma = 30$ demonstrate similar trends in both accuracy and scalability. Kernel multinomial regression with full parameterization runs slightly slower and yields marginally worse log-likelihoods compared to the standard parameterization in this subsection.

\subsection{Increasing the Number of Categories $q$}\label{sec:increaseq}
We have demonstrated in the previous two sections that the efficiency gap between QMME and Newton-Raphson is especially large in multinomial regression, where Newton’s method requires computing and inverting a large Hessian matrix of the form $\sum_{i=1}^n (\Lambda_{\hat{p}}-\hat{p}\hat{p}^\top)\Kron (KG^\top)_i(KG^\top)_i^\top + \lambda I_{q-1} \Kron GKG^\top$, which is of dimension $m(q-1) \times m(q-1)$. In particular, notice that $\sum_{i=1}^n (\Lambda_{\hat{p}}-\hat{p}\hat{p}^\top)\Kron (KG^\top)_i(KG^\top)_i^\top$ does not equal $(\sum_{i=1}^n (\Lambda_{\hat{p}}-\hat{p}\hat{p}^\top))\Kron (\sum_{i=1}^n  (KG^\top)_i(KG^\top)_i^\top)$ in general \citep{bohning1992multinomial}. Inverting such an $m(q-1) \times m(q-1)$ matrix incurs a computational cost of $O(m^3(q-1)^3)$,  so that the computational burden grows dramatically for larger $q$. In contrast, QMME bypasses explicit Hessian inversion by leveraging the Sylvester equation formulation and reusing matrix decompositions, achieving a significantly lower per-iteration complexity of $O(m^2(q-1))$.
\begin{figure}[t]
  \centering
  \includegraphics[width=\textwidth]{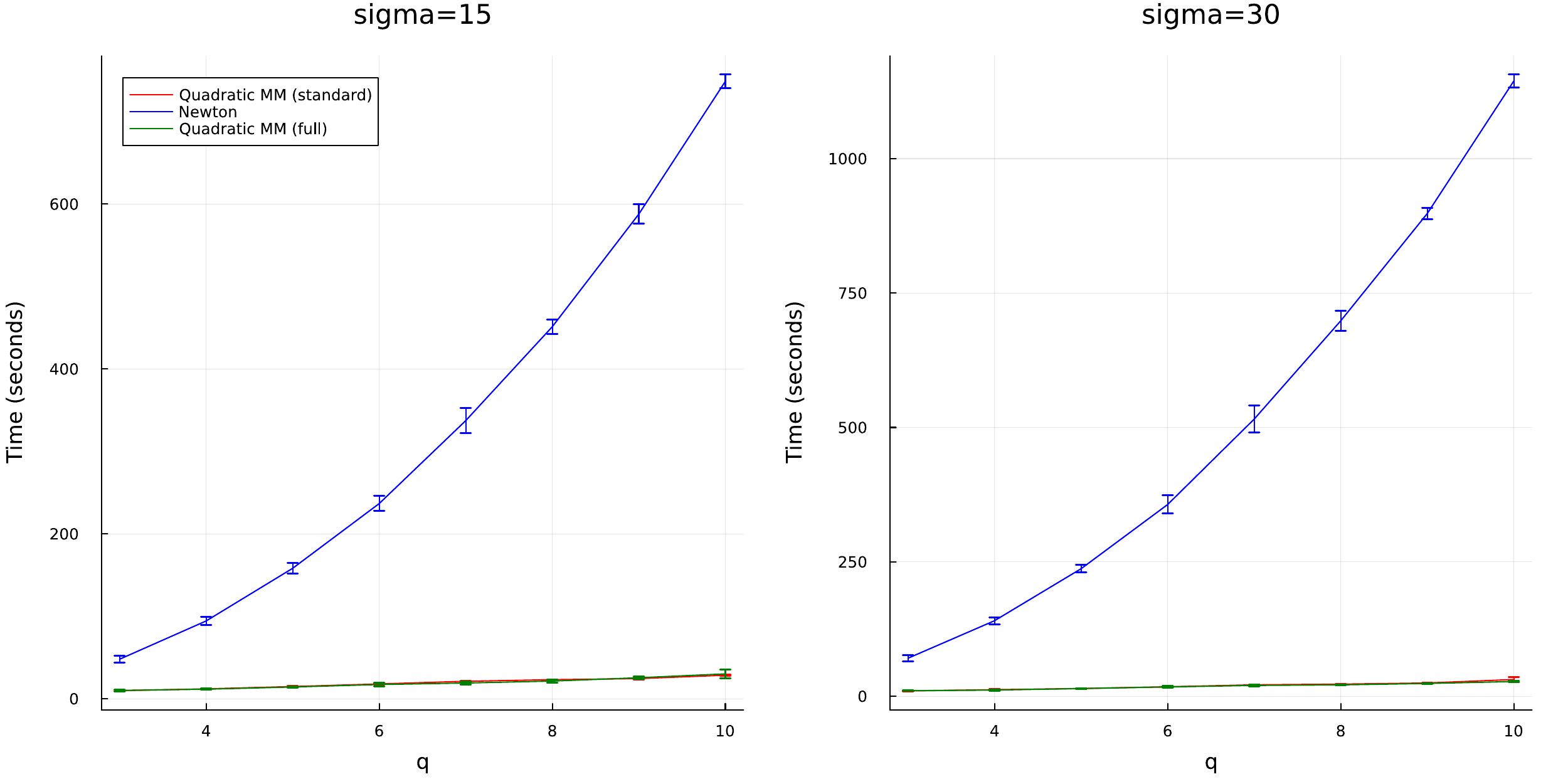}
  \caption{Computation time and best attainable log-likelihood for kernel multinomial regression with increasing $q$. Data points are averaged over 20 random replicates.}
  \label{fig:increaseq}
\end{figure}
In the previous two subsections, we have considered $q=3$, which is in fact conservative in terms of highlighting the computational advantage of QMME over Newton-Raphson. In this section we use a dedicated experiment to examine the relative efficiency of QMME over Newton-Raphson as $q$ increases. Here, $\eta_{1i}$ and $\eta_{2i}$ are still generated in the same way as the beginning of \Sec{ne}. However, here we add in $\eta_{3i}, \eta_{4i},\dots, \eta_{q-1,i}=0$ and compute the probabilities as 
\[
p_{ki} = \frac{\exp(\eta_{ki})}{\sum_{j=1}^{q-1} \exp(\eta_{ji}) + 1}, \quad k = 1,2,\dots,q-1, \qquad p_{qi} = 1 - \sum_{j=1}^{q-1} p_{ji}.
\]
The response $b_i$ is then generated from $\text{Multinomial}[\{1,2,\dots,q\},(p_{1i},p_{2i},\dots, p_{qi})]$. We still use the sample size $n=2^{14}$ but fix $m$ at $m=2^9$. We then vary the number of classes as $q\in\{3,4,5,\dots,10\}$. We again report the total computation time for solving the entire solution path starting from $\lambda=10^1$ to $\lambda=10^{-5}$. As shown in \Fig{increaseq}, the total computation time for Newton-Raphson increases roughly cubically with $q$, whereas QMME shows near-linear scaling. These findings are consistent with our complexity analysis and further demonstrate that QMME becomes increasingly advantageous as the number of classes grows.

\section{Real Data Example}\label{sec:RDA}

To evaluate the performance of our kernel multinomial regression algorithms on real-world data, we use the Codon Usage dataset \citep{hallee2023machine} from the UCI Machine Learning Repository (\url{https://archive.ics.uci.edu/dataset/577/codon+usage}). This dataset contains codon usage frequencies from the coding DNA of a large and diverse set of organisms across multiple taxa, as recorded in the CUTG database. Each instance corresponds to a unique species or genomic entry. The first column denotes the organism’s taxonomic category, such as bct (bacteria), inv (invertebrates), and mam (mammals), for a total of 11 distinct classes. The features, located in columns 6 through 69, represent normalized frequencies of the standard 64 codons, reflecting their relative proportions within an organism’s coding sequences. The dataset includes 13,028 instances and 64 continuous features, making it a compelling test case for high-dimensional, multi-class classification with structured and biologically meaningful predictors. During preprocessing, we impute a small number of missing values in the first two features using column means and then standardize all feature columns.

We adopt the following workflow to fit and evaluate kernel multinomial regression models. The dataset is partitioned into training (70\%), validation (10\%), and test (20\%) subsets.  We set the kernel bandwidth to $\sigma=10$. A solution path of length 30 is computed on the training set, with the regularization parameter $\lambda$ logarithmically spaced from $10^{1}$ to $10^{-5}$. The optimal $\lambda$ is selected based on the log-likelihood evaluated on the validation set. \Fig{illustrate} shows an example of the log-likelihood path evaluated on both the validation and test sets. Notice that the test set has twice as many observations, so the log-likelihoods will naturally be lower than the validation set. Using the combined training and validation data, we then recompute the kernel matrix $K$, resample the sketching matrix $G$, and refit the models using the selected $\lambda$. Model performance is assessed on the test set using classification accuracy and log-likelihood. When generating the sketching matrix $G$, we find that a simple stratified sampling strategy can noticeably improve predictive performance. In particular, we ensure that the sampled rows of $G$ preserve the class proportions in the response variable. This helps prevent underrepresentation of minority classes due to random variation.

\begin{figure}[t]
  \centering
  \includegraphics[width=0.6\textwidth]{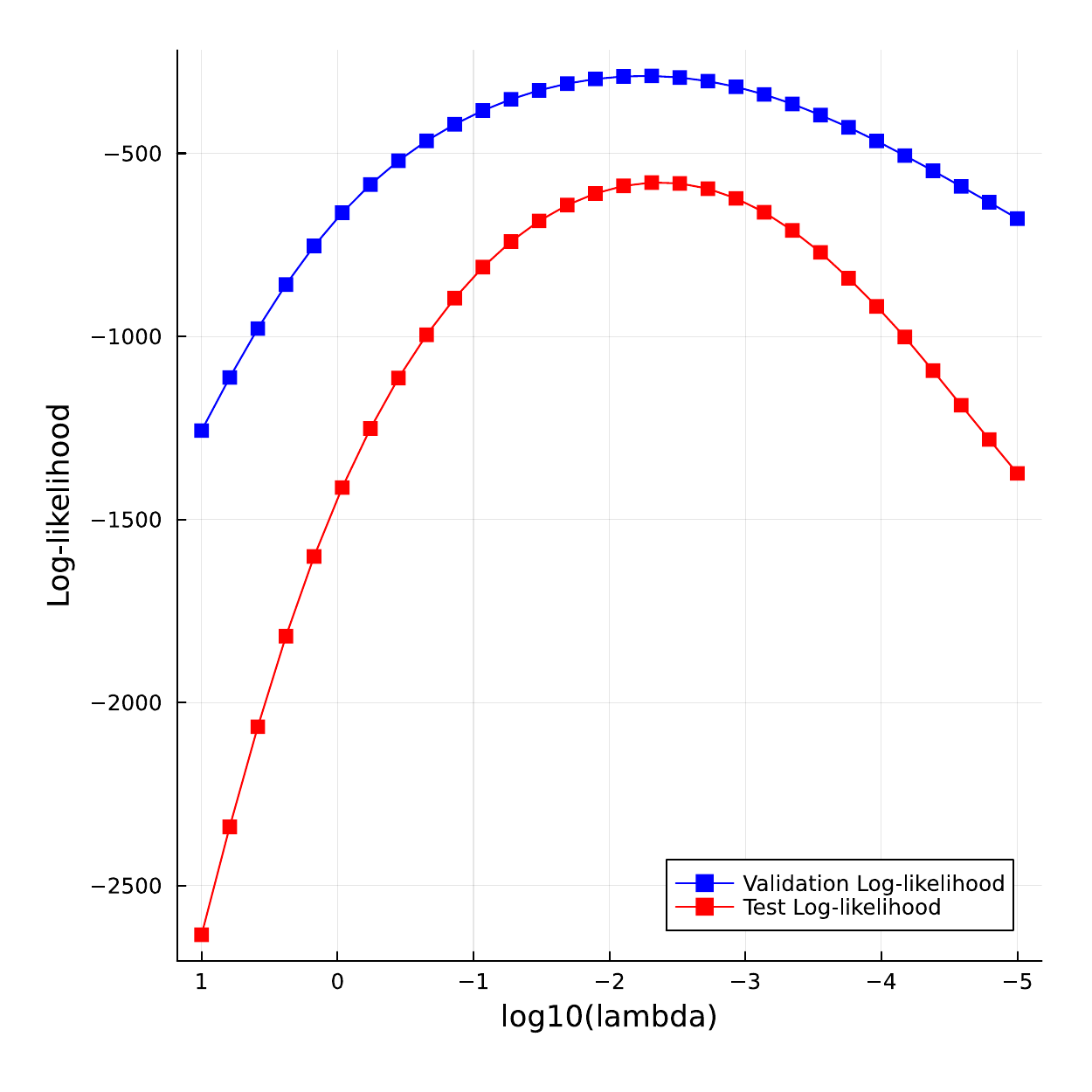}
  \caption{Validation and test log-likelihoods along the regularization path for projection dimension $m = 1024$. }
  \label{fig:illustrate}
\end{figure}

We consider projection dimensions $m \in \{\lfloor n/64 \rfloor, \lfloor n/32 \rfloor, \lfloor n/16 \rfloor, \lfloor n/8 \rfloor, \lfloor n/4 \rfloor\}$ where $n$ is the number of observations in the training set and validation set combined. We compare three methods: QMME with standard parameterization, QMME with full parameterization, and Newton’s method. Due to the prohibitive computational cost, results for Newton’s method at $m = \lfloor n/8 \rfloor$ and $m = \lfloor n/4 \rfloor$ are omitted. The train-validation-test split is repeated 20 times, and all reported metrics are averaged across these random partitions.

From \Tab{method_comparison}, we observe a clear upward trend in both classification accuracy and log-likelihood for all methods as the projection dimension $m$ increases. This suggests that for this relatively high-dimensional dataset ($d=64$), a large projection dimension is crucial for achieving strong performance with kernel methods. The QMME algorithm (with standard parameterization) and Newton’s method yield nearly identical results for the same $m$, which is expected since they optimize the same convex objective function. The kernel multinomial regression model with full parameterization consistently achieves higher log-likelihood values, albeit with a slight sacrifice in classification accuracy. Additionally, a practical advantage of the full parameterization is that it eliminates the need to arbitrarily select a reference category.
\begin{table}[t]
\centering
\caption{Comparison of Kernel Multinomial Regression Methods Across Sketch Sizes}
\label{tab:method_comparison}
\resizebox{\textwidth}{!}{
\begin{tabular}{c l c c c}
\hline
\textbf{Sketch Size $m$} & \textbf{Method}       & \textbf{Time (s)} & \textbf{Log-Likelihood} & \textbf{Accuracy (\%)} \\
\hline
\multirow{3}{*}{$\lfloor n/64 \rfloor$}  & QMME (standard) &   19.49 (1.98)           &   -739.30 (53.00)           &   90.87 (0.82)           \\
                      & QMME (full)     &  18.56 (1.66)             &    -727.93 (52.17)           &   91.03 (0.80)           \\
                      & Newton        &    62.5 (3.82)          &    -739.23 (52.94)            &   90.86 (0.81)           \\
\hline
\multirow{3}{*}{$\lfloor n/32 \rfloor$}  & QMME (standard) &   37.47 (0.97)           &    -644.67 (45.97)            &  92.34 (0.71)            \\
                      & QMME (full)     &   38.20 (1.28)            &  -632.05 (45.54)              &   92.41 (0.74)            \\
                      & Newton        &   245.7 (7.29)           &   -644.69 (45.98)              &     92.34 (0.73)         \\
\hline
\multirow{3}{*}{$\lfloor n/16 \rfloor$}  & QMME (standard) &  88.75 (1.82)            &  -581.02 (44.54)               &  93.56 (0.61)            \\
                      & QMME (full)     &  92.38 (2.18)            &   -568.98 (43.04)             &   93.58 (0.52)           \\
                      & Newton        &    1018.25 (28.97)          &    -580.93 (44.42)            &   93.56 (0.61)            \\
\hline
\multirow{3}{*}{$\lfloor n/8 \rfloor$} & QMME (standard) &     239.33 (7.65)          &  -545.33 (35.98)              & 93.92 (0.38)             \\
                      & QMME (full)     &  251.75 (6.26)            &  -534.56 (35.31)              &  93.89 (0.52)            \\
                      & Newton        & ---          & ---            & ---          \\
\hline
\multirow{3}{*}{$\lfloor n/4 \rfloor$} & QMME (standard) &    804.14 (30.75)          &   -530.32 (37.41)              &  94.14 (0.52)            \\
                      & QMME (full)     &  852.82 (31.10)            &  -521.61 (36.34)              &  94.05 (0.41)           \\
                      & Newton        & ---          & ---            & ---          \\
\hline
\end{tabular}}
\end{table}
In terms of computational efficiency, QMME consistently outperforms Newton’s method, with the performance gap widening as $m$ increases. Notably, when $m$ exceeds $1000$, the computational cost of Newton’s method becomes practically prohibitive. This speed advantage aligns with our earlier discussion in \Sec{increaseq} on computational complexity. To our knowledge, QMME is the first method/framework that is capable of computing large scale kernel multinomial regression problems. 

To provide additional context for the accuracy results reported in \Tab{method_comparison}, we also include the performance of the nearest neighbor (NN) classifier. The NN classifier achieves an average classification accuracy of 93.67\% with a standard deviation of 0.41\%. Despite its simplicity, NN is a surprisingly strong baseline, as also observed in \citet{hallee2023machine}. In that study, using an 80\%/20\% train-test split, NN achieved an accuracy of 96.60\%, while the best-performing model—an ensemble of multiple machine learning methods—achieved 96.95\%. The higher accuracies reported in \citet{hallee2023machine} can be attributed in part to the coarser classification task, which involved only five categories: archaea, bacteria, eukaryote, bacteriophage, and virus. \citet{hallee2023machine} also removed a small proportion of samples in the pre-processing step. Our task involves a more fine-grained classification into 11 distinct categories and includes all of the samples. We are encouraged to observe that kernel multinomial regression, as a single-method approach, is able to outperform NN by nearly 0.5\% in accuracy at $m=\lfloor n/4\rfloor$, which underscores its competitive performance in classification accuracy. Additionally, a crucial advantage of kernel multinomial regression is the ability to provide well-calibrated class probabilities as exemplified by the log-likelihoods.

\section{Conclusion}\label{sec:discuss}

In this work, we develop a Quadratic Majorization Minimization with Extrapolation (QMME) framework that bridges the gap between computationally expensive second-order methods and the slow convergence of first-order approaches. By minimizing a quadratic majorant with fixed curvature at each iteration, QMME achieves a favorable trade-off between per-iteration cost and convergence speed. A key strength of the method lies in its ability to reuse matrix decompositions (e.g., Cholesky, spectral, Schur) across iterations and regularization parameters, significantly enhancing efficiency in large-scale settings.

Our theoretical analysis is carried out with respect to the norm induced by the curvature matrix $H$, leading to several desirable convergence guarantees. Specifically, we establish subsequential convergence as long as the extrapolation step size is bounded away from 1 and global convergence under the Kurdyka-Łojasiewicz (KL) property. Furthermore, by relating QMME to inertial Krasnoselskii-Mann iterations, we show that global convergence can be obtained under relaxed assumptions with additional restrictions on the extrapolation step size.

On the practical side, we demonstrate the efficacy of QMME in kernel-regularized learning. In addition to developing more scalable algorithms for kernel-smoothed quantile regression and kernel logistic regression, we introduce a novel Sylvester equation-based approach for kernel multinomial regression. This technique significantly reduces the dimensionality of the normal equations when the number of classes $q$ is large. Numerical experiments confirm that QMME outperforms established first- and second-order methods (e.g., FISTA, Newton) in speed in the presence of large sketching matrices. The method’s synergy with kernel approximation techniques such as Nyström further advances the scalability of kernel methods. We note that our method may also apply to other kernel based problems or smooth learning problems, such as smooth support vector machines \citep{wang2022density,JMLR:v26:23-1581}.

Our global convergence analysis is primarily based on established frameworks like the KL property. However, we empirically observe that QMME often converges robustly even when the KL condition may not hold, for instance, in multinomial regression models with full parameterization with $\lambda=0$. This motivates the development of new theoretical frameworks imposing weaker assumptions than KL. While the algorithm and analysis presented in this article are applicable to composite optimization, our experiments have focused on smooth problems. It remains an important direction to explore the practical performance of QMME in nonsmooth settings. Promising applications may include $\ell_1$ penalized convolution smoothed quantile regression \citep{JMLR:v24:22-1217,man2024unified}.

% Acknowledgements and Disclosure of Funding should go at the end, before appendices and references

% Manual newpage inserted to improve layout of sample file - not
% needed in general before appendices/bibliography.

\newpage

\appendix

\section{Additional Proofs}\label{sec:proofglobal}
The proof of \Thm{globalconverge} (ii) and \Thm{locallinear} follow arguments in \cite{Wen2018APD}, with the main distinction being that our arguments are built around the induced norm $\|\cdot\|_H$. 
\subsection{Proof of \Thm{globalconverge} (ii)}
\begin{proof}
We note that $E(x^k, x^{k-1})$ is the same as $E_k$ which we defined earlier, so that it converges to $\zeta = \lim_{{k\rightarrow \infty}} f(x^k)$. In the meantime, it is not hard to see that the set of accumulation points sequence $\{(x^k,x^{k-1})\}$ is $\Omega\times\Omega$, where 
$\Omega$ is defined in \Prop{constantcluster}. Consider any $\hat{x}\in \Omega$, we have $E(\hat{x},\hat{x}) = f(\hat{x}) =\zeta$. Thus $E(x,y)\equiv\zeta$ on $\Omega\times \Omega$. 

Next, due to \Thm{subsequenceconverge}, we only need to show that $\{x^k\}$ is a convergent sequence. We first remark on an edge case where $E(x^k,x^{k-1})=\zeta$ for some finite $\bar{k}$. Since $E(x^k,x^{k-1})$ is non-decreasing and convergent to $\zeta$, we conclude that for any $k\ge \bar{k}$, $E(x^k,x^{k-1})=\zeta$, and $x^k = x^{\bar{k}}$ for any $k\ge \bar{k}$.

Next, we consider the case where $E(x^k,x^{k-1})>\zeta$ for all $k$. Under \As{bound}, we know that all level sets of $f$ are bounded, then we have
\begin{eqnarray*}
f(x^k) & \le & f(x^k)+\frac{1}{2}\|x^k-x^{k-1}\|_H^2\le f(x^1)+\frac{1}{2}\|x^1-x^{0}\|_H^2 =  f(x^1),
\end{eqnarray*}
which allows us to conclude that the sequence $\{x^k\}$ is bounded. Then, the set of accumulation points $\Omega$ is compact. Due to \Lem{uniformKL}, there exists $\epsilon,a>0$, and a concave continuous function $\phi$ such that 
\begin{eqnarray*}
\phi'(E(x,y)-\zeta)\text{dist} ((0,0),\partial E(x,y)) & \ge & 1, 
\end{eqnarray*}
for all $(x,y)\in U$, where
\begin{eqnarray*}
U = \{(x,y) \in \Real^n \times \Real^n| \text{dist}((x,y),\Omega\times \Omega)<\epsilon\}\cap \{(x,y) \in \Real^n \times \Real^n|\zeta<E(x,y)<\zeta+a\}.
\end{eqnarray*}
Using the fact that $\{x^k\}$ is bounded and $\Omega\times \Omega$ is the set of accumulation points of $(x^k,x^{k-1})$, we have $\lim_{{k\rightarrow \infty}} \text{dist}((x^k,x^{k-1}),\Omega\times\Omega)=0$. Additionally, $E(x^k,x^{k-1})$ is non-increasing and convergent to $\zeta$. Thus there exists $k_0$ such that for all $k\ge k_0$, we have $(x^k,x^{k-1})\in U$. so that
\begin{eqnarray*}
\forall k\ge k_0, && \phi'(E(x^k,x^{k-1})-\zeta)\text{dist}((0,0),\partial E(x^k,x^{k-1})) \ge 1.
\end{eqnarray*}
Using the concavity of $\phi$, we have for any $k\ge k_0$, 
\begin{align*}
&[(\phi (E(x^k,x^{k-1})-\zeta)-(\phi (E(x^{k+1},x^{k})-\zeta)]\text{dist}((0,0),\partial E(x^k,x^{k-1}))\\
&\ge\phi' (E(x^k,x^{k-1})-\zeta) \text{dist}((0,0),\partial E(x^k,x^{k-1})) [E(x^k,x^{k-1})-E(x^{k+1},x^{k})]\\
& \ge E(x^k,x^{k-1})-E(x^{k+1},x^{k})
\end{align*}
Due to \Eqn{diffofE}, there exists a positive number $D$ such that
\begin{eqnarray*}
E(x^k,x^{k-1})-E(x^{k+1},x^{k}) &\ge& D \|x^k -x^{k-1}\|_H^2, 
\end{eqnarray*}
then we have
\begin{align*}
\|x^k -x^{k-1}\|_H^2 \le \frac{C}{D}(\phi (E(x^k,x^{k-1})-\zeta)-(\phi (E(x^{k+1},x^{k})-\zeta)) (\|x^k - x^{k-1}\|_H+\|x^{k-1}-x^{k-2}\|_H)
\end{align*}
Taking square roots on both sides and using the AM-GM inequality, we have
\begin{eqnarray*}
\|x^k -x^{k-1}\|_H &\le & \sqrt{\frac{2C}{D}(\phi (E(x^k,x^{k-1})-\zeta)-(\phi (E(x^{k+1},x^{k})-\zeta))}\\
&&\cdot\sqrt{\frac{\|x^k - x^{k-1}\|_H+\|x^{k-1}-x^{k-2}\|_H}{2}}\\
& \le &\frac{C}{D}(\phi (E(x^k,x^{k-1})-\zeta)-(\phi (E(x^{k+1},x^{k})-\zeta))\\&&+\frac{\|x^k - x^{k-1}\|_H+\|x^{k-1}-x^{k-2}\|_H}{4},
\end{eqnarray*}
so that
\begin{eqnarray*}
\frac{1}{2} \|x^k -x^{k-1}\|_H  &\le  \frac{C}{D}(\phi (E(x^k,x^{k-1})-\zeta)-(\phi (E(x^{k+1},x^{k})-\zeta))\\
&+ \frac{1}{4} (\|x^{k-1}-x^{k-2}\|_H -\|x^{k}-x^{k-1}\|_H),
\end{eqnarray*}
which implies that 
\begin{eqnarray*}
\sum_{k=k_0}^\infty \| x^k - x^{k-1}\|_H & \le & \frac{2C}{D}\phi (E(x^{k_0},x^{{k_0}-1})-\zeta) + \frac{1}{2}\|x^{{k_0}-1}-x^{{k_0}-2}\|_H< \infty,
\end{eqnarray*}
which implies the convergence of $x^k$. 
\end{proof}
\subsection{Proof of \Thm{locallinear}}\label{sec:locallinear}
\begin{proof}
Following the proof of \Thm{globalconverge} (ii), denote $\sum_{i=k}^\infty \| x^{i+1}- x^{i}\|_H$ as $S_k$,   for any $k\ge k_0$ we have
\begin{eqnarray}\label{eq:Strelation}
S_k \le S_{k-1}& \le & \frac{2C}{D}\phi (E(x^{k},x^{{k}-1})-\zeta) + \frac{1}{2}\|x^{{k}-1}-x^{{k}-2}\|_H \nonumber \\
& =  &\frac{2C}{D}\phi (E(x^{k},x^{{k}-1})-\zeta) +  \frac{1}{2} (S_{k-2}-S_{k-1}) \nonumber\\
& \le & \frac{2C}{D}\phi (E(x^{k},x^{{k}-1})-\zeta) +  \frac{1}{2} (S_{k-2}-S_{k})
\end{eqnarray}
Using the KL inequlaity for $E(x,y)$, we have
\begin{eqnarray*}
 \frac{1}{\sqrt{2\mu_E}} (E(x^{k},x^{{k}-1})-\zeta)^{-1/2}\text{dist}((0,0),\partial E(x^k,x^{k-1})) &\ge & 1.
\end{eqnarray*}
Due to \Eqn{distancebound}, for sufficiently large $k$,
\begin{eqnarray*}
\text{dist}((0,0),\partial E(x^k,x^{k-1}))&\le& C(S_{k-2}-S_{k}),
\end{eqnarray*}
thus, for sufficiently large $k$,
\begin{eqnarray*}
(E(x^{k},x^{{k}-1})-\zeta)^{1/2} & \le &\frac{1}{\sqrt{2\mu_E}} \cdot C \cdot (S_{k-2}-S_{k}).
\end{eqnarray*}
Plugging into \Eqn{Strelation}, recall that the KL function is \(\phi(s) = \frac{\sqrt{2}}{\sqrt{\mu_E}} s^{1/2}\), so that for sufficiently large $k$, 
\begin{eqnarray*}
S_k & \le & \frac{2C^2}{D\mu_E} (S_{k-2}-S_{k}) + S_{k-2}-S_{k} = (\frac{2C^2}{D\mu_E}+1)( S_{k-2}-S_{k}),
\end{eqnarray*}
Let $C_1 = \frac{2C^2}{D\mu_E}$, there exists $k_1>0$ such that for all $k\ge k_1$
\begin{eqnarray*}
S_k & \le & \frac{C_1+1}{C_1+2} S_{k-2},
\end{eqnarray*}
Then for all $k\ge k_1$
\begin{eqnarray*}
\|x^k - \bar{x}\|_H \le \sum_{i=k}^\infty \|x^{k+1} - x^k\|_H = S_{k}&\le &S_{k_1-2}\left(\sqrt{\frac{C_1+1}{C_1+2}}\right)^{k-k_1+1},
\end{eqnarray*}
so that
\begin{eqnarray*}
\|x^k - \bar{x}\| & \le & \frac{S_{k_1-2}}{\sqrt{\sigma_{\min}(H)}}\left(\sqrt{\frac{C_1+1}{C_1+2}}\right)^{k-k_1+1}.
\end{eqnarray*}
 
\end{proof}
\section{Generalization to Composite Optimization}
In this section, we generalize the QMME framework to minimize composite functions of the form $F(x) = f(x)+g(x)$. $f(x)$ satisfies the same assumptions we make for smooth optimization, while $g(x)$ is lower semicontinuous and convex but possibly non-smooth. QMME for minimizing $F(x)$ proceeds via the following iterations:
\begin{eqnarray}\label{eq:MMextrapolate2}
y^k & = &x^k + \beta_k (x^k - x^{k-1}),\nonumber\\
x^{k+1} & = &\underset{x\in \Real^n}{\text{argmin}} \quad \langle \nabla f(y^k) , x\rangle +\frac{1}{2}\| y^k - x \|_H^2 + g(x).
\end{eqnarray}
The update in \Eqn{MMextrapolate2} corresponds to a generalized proximal operator, and the resulting QMME scheme is closely related to proximal Newton methods \citep{lee2014proximal}. Evaluating this generalized proximal operator is generally nontrivial and typically requires an iterative subroutine, such as coordinate descent for $\ell_1$-penalized generalized linear models \citep{friedman2010regularization}. 

Although the algorithm and theory developed for \Eqn{MMextrapolate} generalize to the composite case, the practical advantages of QMME over classical proximal Newton methods warrant further investigation. In particular, the use of a quadratic surrogate with fixed curvature matrix $H$ in \Eqn{MMextrapolate2} may not lead to substantially cheaper iterations compared to directly using the original Hessian. However, the QMME framework may enjoy more robust convergence. We leave the empirical comparison between QMME and Proximal-Newton as future work. In what follows, we sketch the convergence analysis, highlighting only the components that differ from the smooth setting.

Unlike smooth optimization, the optimal set $\mathcal{X}$ changes to $\{x\in \Real^n|0\in \nabla f(x)+\partial g(x)\}$. The statement and proof for \Lem{strongconvexity} need to be adapted as follows:
\begin{lemma}
We have for any $z\in \Real^n$,
\begin{eqnarray}\label{eq:strongconvexity2}
&& \langle \nabla f(y^k),x^{k+1}\rangle + \frac{1}{2} \| y^k - x^{k+1} \|_H^2  + g(x^{k+1}) +\frac{1}{2} \| z - x^{k+1} \|_H^2    \nonumber\\& \le & \langle \nabla f(y^k) , z\rangle +\frac{1}{2}\| y^k - z \|_H^2 + g(z).
\end{eqnarray}
\end{lemma}

\begin{proof}
Due the optimality condition of minimization problem \Eqn{MMextrapolate2}, we have
\begin{eqnarray*}
0 \in \nabla f (y^k) + H(x^{k+1} - y^k) + \partial g(x^{k+1}),
\end{eqnarray*}
in other words there exists $v\in \partial g(x^{k+1})$ such that
\begin{eqnarray*}
\nabla f (y^k) + H(x^{k+1} - y^k) + v = 0.
\end{eqnarray*}
Using the subgradient inequality for $g$ at $x^{k+1}$, we have
\begin{eqnarray*}
g(z) &\ge & g(x^{k+1}) + \langle v, z - x^{k+1}).
\end{eqnarray*}
Sustitute $v = - \nabla  f (y^k) - H(x^{k+1} - y^k)$ into the inequality gives us 
\begin{eqnarray}\label{eq:subgrad}
g(z) \ge g(x^{k+1}) + \langle - \nabla  f (y^k) - H(x^{k+1} - y^k), z - x^{k+1})
\end{eqnarray}
Adding \Eqn{subgrad} and \Eqn{expand} gives us 
\begin{eqnarray}\label{eq:rearrange2}
g(z) + \frac{1}{2}\|y^k -z\|_H^2 & \ge & g(x^{k+1}) + \langle - \nabla  f (y^k), z - x^{k+1}\rangle + \frac{1}{2} \| y^k - x^{k+1} \|_H^2\\&& + \frac{1}{2}\| x^{k+1}-z\|_H^2. \nonumber
\end{eqnarray}
Rearranging terms of \Eqn{rearrange2} gives us exactly \Eqn{strongconvexity2}. 
\end{proof}
The auxillary sequence $E_k$ should now be defined as $E_k = F(x^k)+\frac{1}{2} \|x^k - x^{k-1}\|_H^2$ and we can prove \Lem{aux} simply by replacing $f(x)$ with $F(x)$ in \Eqn{boundnext}, \Eqn{boundf},\Eqn{difff},\Eqn{diffofE}. In the statement of \Thm{subsequenceconverge}, $f$ should be replaced with $F$. The proof of  (ii) is now:
\begin{proof}
Due to the optimality condition of \Eqn{MMextrapolate2}, we have
\begin{eqnarray*}
0 \in \nabla f(y^{k_i}) + H (x^{k_i+1} - y^{k_i}) + \partial g(x^{k_i+1}). 
\end{eqnarray*}
Combined with the definition of $y^{k_i}$, we have
\begin{eqnarray*}
-H (x^{k_i+1} - x^{k_i}-\beta_k (x^{k_i} - x^{k_i-1})) & \in & \nabla f(y^{k_i}) + \partial g(x^{k_i+1})
\end{eqnarray*}
Taking $i$ to the limit, due to the Lipschitz continuity of $\nabla f$ and the closedness of the multifunction $\partial g$ (see page 80, Exercise 8 of \cite{jonathan2006convex}), we have
\begin{eqnarray*}
0 &\in& \nabla f(\bar{x}) + \partial g(\bar{x}). 
\end{eqnarray*}
Thus, any accumulation point of $\{x^k\}$ is a stationary point for $F$.
\end{proof}
Similarly, $f$ should become $F$ in \Prop{constantcluster} and the proof of (ii) becomes slightly more involved.
\begin{proof}
For any $\hat{x}\in \Omega$, there is a subsequence $\{x^{k_i}\}$ such that $x^{k_i}\rightarrow \hat{x}$. Since $x^{k_i+1}$ is the minimizer of \Eqn{MMextrapolate}, one has
\begin{eqnarray}\label{eq:minimizer}
&&g(x^{k_i}) + \langle \nabla f(y^{k_i-1}), x^{k_i}\rangle  + \frac{1}{2}\|x^{k_i} - y^{k_i-1}\|_H^2\\&\le& g(\hat{x}) + \langle \nabla f(y^{k_i-1}), \hat{x}\rangle  + \frac{1}{2}\|\hat{x} - y^{k_i-1}\|_H^2. \nonumber
\end{eqnarray}
We note that due to \Thm{subsequenceconverge}, $x^{k_i} - y^{k_i-1} = x^{k_i}- x^{k_i-1} - \beta^{k_i-1}(x^{k_i-1}- x^{k_i-2}) \rightarrow 0$, and $\hat{x} - y^{k_i-1} = \hat{x}- x^{k_i-1} - \beta^{k_i-1}(x^{k_i-1}- x^{k_i-2}) \rightarrow 0$. Taking $i$ to infinity on both sides of \Eqn{minimizer} gives us
\begin{eqnarray*}\label{eq:minimizer2}
\zeta & = & \underset{i\rightarrow \infty}\lim f(x^{k_i})+ g(x^{k_i})\\
      & = & \underset{i\rightarrow \infty}\lim f(x^{k_i})+ g(x^{k_i}) + \langle \nabla f(y^{k_i-1}), x^{k_i}-\hat{x}\rangle  + \frac{1}{2}\|x^{k_i} - y^{k_i-1}\|_H^2\\
      & \le & \underset{i\rightarrow \infty}\limsup f(x^{k_i}) + g(\hat{x}) +  \frac{1}{2}\|\hat{x} - y^{k_i-1}\|_H^2 = F(\hat{x}) \le\underset{i\rightarrow \infty}\liminf F(x^{k_i}) = \zeta.
\end{eqnarray*}
where the first inequality follows from the \Eqn{minimizer} and the second inequality follows from the lower semicontinuity of $F$. 
\end{proof}
Finally, for \Thm{globalconverge}, $f$ should again be replaced with $F$. The proof sketch is given as follows:
\begin{proof}
To establish the result, we need to change the definition of $E(x,y)$ to $E(x,y) = F(x)+\frac{1}{2}\|x-y\|_H^2+g(x)$. Then $\partial E(x^k, x^{k-1}) = (\nabla f(x^k)+H(x^k-x^{k-1})+\partial g(x^k), - H(x^k - x^{k-1}))$. Using the optimality condition of \Eqn{MMextrapolate2} we have
\begin{eqnarray*}
0 &\in &\nabla f(y^{k-1}) + H (x^{k} - y^{k-1}) + \partial g(x^{k}). 
\end{eqnarray*}
so that
\begin{eqnarray*}
 (\nabla f(x^k) - \nabla f(y^{k-1}) -H(x^{k-1}-y^{k-1}), - H(x^k - x^{k-1}) ) & \in & \partial E(x^k,x^{k-1}).
\end{eqnarray*}
The remaining parts of the proof are nearly identical.
\end{proof}

\bibliography{refs}

\end{document}